\documentclass[10pt,twocolumn,letterpaper]{article}

\usepackage{cvpr}
\usepackage{times}
\usepackage{epsfig}
\usepackage{graphicx}
\usepackage{amsmath}
\usepackage{amssymb}
\usepackage{amsthm}
\usepackage[pagebackref=true,breaklinks=true,letterpaper=true,colorlinks,bookmarks=false]{hyperref}

\makeatletter
\@namedef{ver@everyshi.sty}{}
\makeatother

\usepackage{adjustbox}
\usepackage{bm}
\usepackage{subcaption}
\usepackage{threeparttable}

\usepackage{array}
\usepackage{mdwmath}
\usepackage{mdwtab}
\usepackage{eqparbox}
\usepackage{fixltx2e}
\usepackage{stfloats}
\usepackage{url}
\usepackage{diagbox}
\usepackage{verbatim}
\usepackage{booktabs}
\usepackage{multirow}
\usepackage{mathtools}
\usepackage{breqn}
\def\etal{{\it et al. }}
\theoremstyle{definition}

\newtheorem{theorem}{Theorem}
\newtheorem{corollary}{Corollary}[theorem]

\cvprfinalcopy 

\def\cvprPaperID{1655} 

\ifcvprfinal\fi

\begin{document}

\title{Mitigating Information Leakage in Image Representations: A Maximum Entropy Approach} 

\author{Proteek Chandan Roy and Vishnu Naresh Boddeti \\
Department of Computer Science and Engineering \\
Michigan State University, East Lansing MI 48824\\
{\tt \{royprote, vishnu\}@msu.edu}
}

\maketitle
\thispagestyle{empty}


\begin{abstract}

Image recognition systems have demonstrated tremendous progress over the past few decades thanks, in part, to our ability of learning compact and robust representations of images. As we witness the wide spread adoption of these systems, it is imperative to consider the problem of unintended leakage of information from an image representation, which might compromise the privacy of the data owner. This paper investigates the problem of learning an image representation that minimizes such leakage of user information. We formulate the problem as an adversarial non-zero sum game of finding a good embedding function with two competing goals: to retain as much task dependent discriminative image information as possible, while simultaneously minimizing the amount of information, as measured by entropy, about other sensitive attributes of the user. We analyze the stability and convergence dynamics of the proposed formulation using tools from non-linear systems theory and compare to that of the corresponding adversarial zero-sum game formulation that optimizes likelihood as a measure of information content. Numerical experiments on UCI, Extended Yale B, CIFAR-10 and CIFAR-100 datasets indicate that our proposed approach is able to learn image representations that exhibit high task performance while mitigating leakage of predefined sensitive information.
\end{abstract}

\section{Introduction}
Current day machine learning algorithms based on deep neural networks have demonstrated impressive progress across multiple domains such as image classification, speech recognition etc. By stacking together multiple layers of linear and non-linear operations deep neural networks have been able to learn and identify complex patterns in data. As a by-product of these capabilities, deep neural networks have also become powerful enough to inadvertently identify sensitive information or features of data even in the absence of any additional side information. For example, consider a scenario where a user enrolls their facial image in a face recognition system for the purpose of access control. During enrollment, a feature vector is extracted from the image and stored in a database. Apart from the identity of the user, this feature vector potentially contains information that is sensitive to the user, such as the age, information that the user may never have expressly consented to provide. More generally, learned data representations could leak auxiliary information that the participants may never have intended to release. Information obtained in this manner can be used to compromise the privacy of the user or to be biased and unfair to the user. Therefore, it is imperative to develop representation learning algorithms that can \emph{intentionally} and \emph{permanently} obscure sensitive information while retaining task dependent information. Addressing this problem is the central aim of this paper.

A few recent attempts have been made to study related problems, such as learning censored \cite{edwards2015censoring}, fair \cite{louizos2015variational}, or invariant \cite{xie2017controllable} representations of data. The central idea of these approaches, collectively referred to as \emph{Adversarial Representation Learning} (ARL), is to learn a representation of data in an adversarial setting. These approaches couple together (i) an adversarial network that seeks to classify and extract sensitive information from a given representation, and (ii) an embedding network that is tasked with extracting a compact representation of data while preventing the adversarial network from succeeding at leaking sensitive information. To achieve their respective goals, the adversary is optimized to maximize the likelihood of the sensitive information, while the encoder is optimized to minimize the same likelihood i.e., adversary's likelihood of the sensitive information, thereby leading to a zero-sum game. We will henceforth refer to this formulation as \emph{Maximum Likelihood Adversarial Representation Learning} (ML-ARL).

The zero-sum game formulation of optimizing the likelihood, however, is practically sub-optimal from the perspective of preventing information leakage. As an illustration consider a problem where the sensitive attribute has three categories. Let there be two instances where the adversary's probability distribution of the sensitive label is (0.33, 0.17, 0.5) and (0.33., 0.33., 0.33.) and let the correct label be class 1 for both of them. In each of these cases the likelihood of the discriminator is the same i.e., $\log 0.33$ but the former instance is more informative than the latter. Moreover, the potential of this formulation to prevent information leakage is predicated upon: (i) the existence of an equilibrium, and (ii) the ability of practical optimization procedures to converge to such an equilibrium. As we will show, in practice, the conditions necessary for convergence may not be satisfied. Therefore, when the optimization does not reach the equilibrium, a probability distribution with the minimum likelihood is the distribution that is most certain with the potential to leak the most amount of information. In contrast, the second instance is a uniform distribution over the sensitive labels and provides no information to the adversary. This solution corresponds to the maximum entropy distribution over the sensitive labels.

\vspace{5pt}
\noindent\textbf{Contributions:} Building on the observations above, we propose a framework, dubbed \emph{Maximum Entropy Adversarial Representation Learning} (MaxEnt-ARL), which optimizes an image representation with two major objectives, (i) maximally retain information pertinent to a given target attribute, and (ii) minimize information leakage about a given sensitive attribute. We pose the learning problem in an adversarial setting as a non-zero sum three player game between an encoder, a predictor and a discriminator (proxy adversary) where the encoder tries to maximize the entropy of the discriminator on the sensitive attribute and maximizes the likelihood of the predictor on the target attribute.

We analyze the equilibrium and convergence properties of the ML-ARL as well as the proposed MaxEnt-ARL formulation using tools from non-linear systems theory. We compare and evaluate the numerical performance of ML-ARL and MaxEnt-ARL for fair classification tasks on the UCI dataset, illumination invariant classification on the Extended Yale B dataset and two fabricated tasks on the CIFAR-10 and CIFAR-100 datasets. On a majority of these tasks MaxEnt-ARL outperforms all other baselines.

\section{Related Work}
\vspace{5pt}
\noindent \textbf{Adversarial Representation Learning:} In the context of image classification, adversarial learning has been utilized to learn representations that are invariant across domains \cite{ganin2015unsupervised, ganin2016domain, tzeng2017adversarial}, thereby enabling us to train classifiers on a source domain and utilize on a target domain.

The entire body of work devoted to learning fair and unbiased representations of data share many similarities to the adversarial representation learning problem. Early work on this topic did not involve an explicit adversary but shared the goal of learning representations with competing objectives. The concept of learning fair representations was first introduced by Zemel et al~\cite{zemel2013learning}, where the goal was to learn a representation of data by ``fair clustering'' while maintaining the discriminative features of the prediction task. Building upon this work many approaches have been proposed to learn an unbiased representation of data while retaining its effectiveness for a prediction task. To remove influence of ``nuisance variables" Louizos et al \cite{louizos2015variational} proposed variational fair autoencoder (VFAE), a joint optimization framework for learning an invariant representation and a prediction task. In order to improve fairness in the representation, they regularized the marginal distribution $p(z|s)$ through Maximum Mean Discrepancy (MMD).

More recent approaches \cite{edwards2015censoring,zhang2018mitigating,beutel2017data,xie2017controllable} have used explicit adversarial networks to measure information content of sensitive attributes. These problems are set up as a minimax game between the encoder and the adversary. The encoder is setup to achieve fairness by maximizing the loss of the adversary i.e. minimizing negative log-likelihood of sensitive variables as measured by the adversary. Among these approaches, our proposed MaxEnt-ARL formulation is most directly related to the Adversarial Invariant Feature Learning introduced by Xie et al. \cite{xie2017controllable}.

\vspace{5pt}
\noindent\textbf{Optimization Theory for Adversarial Learning:} The formulation of adversarial representation learning poses unique challenges from an optimization perspective. The parameters of the models in ARL are typically optimized through stochastic gradient descent, either jointly \cite{edwards2015censoring, mescheder2017numerics} or alternatively \cite{ganin2015unsupervised}. The former is, however, more commonly used in practice and is a generalization of gradient descent. While the convergence properties of gradient descent and its variants are well understood, there is relatively little work on the convergence and stability of simultaneous gradient descent in adversarial minimax problems. Recently, Mescheder \etal~\cite{mescheder2017numerics} and Nagarajan \etal~\cite{nagarajan2017gradient} both leveraged tools from non-linear systems theory \cite{khalil1996nonlinear} to analyze the convergence properties of simultaneous gradient descent in the context of GANs. They show that without the introduction of additional regularization terms to the objective of the zero-sum game, simultaneous gradient descent does not converge. Our convergence analysis of ML-ARL and MaxEnt-ARL also leverages the same non-linear systems theory tools and show the conditions under which they converge.

\begin{figure*}[!ht]
    \centering
    \includegraphics[width=0.6\textwidth]{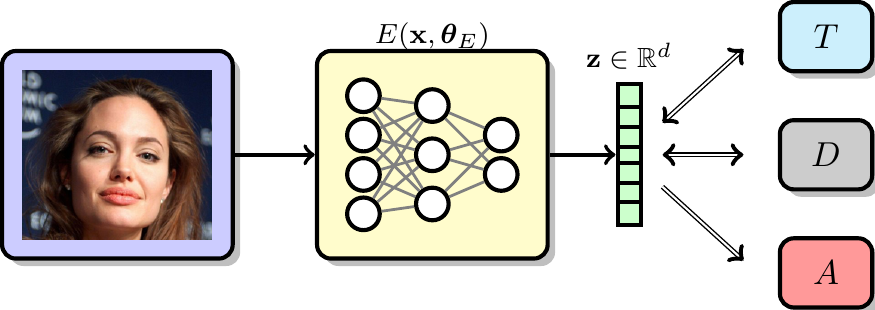}
    \caption{\textbf{Adversarial Representation Learning:} We consider the problem of learning an embedding function $E(\cdot, \bm{\theta}_E)$ that maps a high-dimensional image to a low-dimensional representation $\mathbf{z}\in\mathbb{R}^d$ while satisfying two competing goals: retain as much image information necessary to accurately predict a target attribute $\mathcal{Y}=\{y_1,\dots,y_n\}$ while simultaneously minimizing information leakage about a sensitive attribute $\mathcal{S}=\{s_1,\dots,s_m\}$ by an unknown adversary $A(\cdot, \bm{\theta}_A)$. The learning problem is formulated as a minimax game between $\{E(\cdot, \bm{\theta}_E), T(\cdot,\bm{\theta}_T)\}$ and a proxy adversary $D(\cdot,\bm{\theta}_D)$. \label{fig:overview}}
\end{figure*}

\section{Adversarial Representation Learning}

The Adversarial Representation Learning setup involves observational input $\bm{x}$, a target attribute with $n$ classes $\mathcal{Y}=\{y_1, \dots,\allowbreak y_n\}$ and a sensitive attribute with $m$ classes $\mathcal{S}=\{s_1,\ldots, s_m\}$. In this paper, we restrict ourselves to attributes over a discrete space with multiple labels. Our goal is to learn an embedding function that maps $\bm{x}$ to $\bm{z}$ from which we can predict a target attribute $\mathcal{Y}$, while also minimizing information leakage about a known sensitive attribute $\mathcal{S}$ i.e. class labels of attribute $\mathcal{S}$.

\subsection{Problem Setting}
The Adversarial Representation Learning problem is formulated as a game among three players, \emph{encoder} $E$, a \emph{target predictor} $T$, and a \emph{discriminator} $D$ that serves as a proxy for an unknown \emph{adversary} $A$. After $E$ is learned and fixed, we train and evaluate an \emph{adversary} $A$ with the aim of leaking information of the sensitive attribute that we sought to protect. Since the adversary $A$ is unknown to encoder at training, the encoder $E$ is trained against the discriminator $D$, which thereby acts as a proxy for the unknown $A$. An illustration of this setting is shown in Fig. \ref{fig:overview}. The \emph{encoder} is modeled as a deterministic function, $\bm{z}=E(\bm{x};\bm{\theta}_{E})$, the \emph{target predictor} models the conditional distribution $p(t|\bm{x})$ via $q_T(t|\bm{z}; \bm{\theta}_T)$ and the \emph{discriminator} models the conditional distribution $p(s|\bm{x})$ via $q_D(s|\bm{z};\bm{\theta}_D)$, where $p(t|\bm{x})$ and $p(s|\bm{x})$ are the ground truth labels for a given target and sensitive labels $t$ and $s$, respectively.

\subsection{Background}
In existing formulations of ARL, the goal of the encoder is to maximize the likelihood of the target attribute, as measured by the \emph{target predictor} $T$, while minimizing the likelihood of the sensitive attribute, as measured by the \emph{discriminator} $D$. This problem (henceforth referred to as ML-ARL) was formally defined by Xie \etal \cite{xie2017controllable} as a three player  zero-sum minimax game:
\begin{equation}
    \min_{\bm{\theta}_E,\bm{\theta}_T}\max_{\bm{\theta}_D} J_1(\bm{\theta}_E,\bm{\theta}_T) -\alpha J_2(\bm{\theta}_E,\bm{\theta}_D)
    \label{eq:likelihood}
\end{equation}
\noindent where $\alpha$ is a parameter that allows us to trade-off between the two competing objectives for the encoder and, 
\begin{align}
    J_1(\bm{\theta}_E,\bm{\theta}_T) &= KL\left(p\left(t|\bm{x}\right)\|q_T\left(t|E(\bm{x};\bm{\theta}_E);\bm{\theta}_T\right)\right) \nonumber \\
    J_2(\bm{\theta}_E,\bm{\theta}_D) &= KL\left(p\left(s|\bm{x}\right)\|q_D\left(s|E(\bm{x};\bm{\theta}_E);\bm{\theta}_D\right)\right) \nonumber
\end{align}
\noindent where the $KL(\cdot\|\cdot)$ terms reduce to the log-likelihood if the label distributions are ideal categorical distributions.

\subsection{Maximum Entropy Adversarial Representation Learning}
In the MaxEnt-ARL formulation the goal of the encoder is to maximize the likelihood of the target attribute, as measured by the \emph{target predictor}, while maximizing the uncertainty in the sensitive attribute, as measured by the entropy of the \emph{discriminator's} prediction. Formally, we define the MaxEnt-ARL optimization problem as a three player  non-zero sum game:
\begin{equation}
    \label{eq:maxent}
    \begin{aligned}
    \min_{\bm{\theta}_D} & \mbox{ }  V_1(\bm{\theta}_E,\bm{\theta}_D) \\ 
    \min_{\bm{\theta}_E,\bm{\theta}_T} & \mbox{ } V_2(\bm{\theta}_E,\bm{\theta}_T) +\alpha V_3(\bm{\theta}_E,\bm{\theta}_D) 
    \end{aligned}
\end{equation}
\noindent where $\alpha$ allows us to trade-off between the two competing objectives for the encoder and,
\begin{align}
    V_1(\bm{\theta}_E,\bm{\theta}_D) &= KL\left(p\left(s|\bm{x}\right)\|q_D\left(s|E(\bm{x};\bm{\theta}_E);\bm{\theta}_D\right)\right) \nonumber \\
    V_2(\bm{\theta}_E,\bm{\theta}_T) &= KL\left(p\left(t|\bm{x}\right)\|q_T\left(t|E(\bm{x};\bm{\theta}_E);\bm{\theta}_T\right)\right) \nonumber \\
    V_3(\bm{\theta}_E,\bm{\theta}_D) &= KL\left(q_D\left(s|E(\bm{x};\bm{\theta}_E\right);\bm{\theta}_D)\|U\right) \nonumber
\end{align}
\noindent where $U$ is the uniform distribution. The crucial difference between the MaxEnt-ARL formulation and the ML-ARL formulation is the fact that while the encoder and the discriminator have competing objectives, in ML-ARL they directly compete against each other on the same metric (likelihood of sensitive attribute), while in MaxEnt-ARL they are optimizing competing metrics that are related but not the exact same metric.

Optimizing the embedding function to maximize the entropy of the discriminator instead of minimizing its likelihood has one crucial practical advantage. Entropy maximization inherently does not need class labels for training. This is advantageous in settings where it is either, (i) Undesirable for the embedding function to have access to the sensitive label, potentially for privacy reasons., or (ii) Sensitive labels for the data points are unknown. For instance consider, a semi-supervised scenario where only the desired label is known while the sensitive label is unknown. The embedding function can learn from such data by obtaining gradients from the entropy of the discriminator.

\section{Theoretical Analysis \label{sec:theory}}

In this section we analyze the properties of the MaxEnt-ARL formulation and compare it to the ML-ARL formulation, both  in terms of equilibrium as well as convergence dynamics under simultaneous gradient descent.

\subsection{Equilibrium}
\begin{theorem}\label{th1}
Given a fixed encoder $E$, the optimal discriminator is $q_D(s|E(\bm{x};\bm{\theta}_E);\bm{\theta}_D^{*})=p(s|E(\bm{x};\bm{\theta}_E))$ and the optimal predictor is $q_T(t|E(\bm{x};\bm{\theta}_E);\bm{\theta}_T^{*})=p(t|E(\bm{x};\bm{\theta}_E))$.
\end{theorem}
\begin{proof} The proof uses the fact that, given a fixed encoder $E$, the objective is convex w.r.t. each distribution. Thus we can obtain the stationary point for $q_D(s|E(\bm{x};\bm{\theta}_E);\bm{\theta}_D)$ and $q_T(s|E(\bm{x};\bm{\theta}_E);\bm{\theta}_T)$ as a function of $p(s|E(\bm{x};\bm{\theta}_E))$ and $p(t|E(\bm{x};\bm{\theta}_E))$, respectively. The detailed proof is included in the supplementary material.
\end{proof} 

Therefore, both the optimal distributions $q_D(s|E(\bm{x};\bm{\theta}_E);\bm{\theta}_D^{*})$ and $q_T(t|E(\bm{x};\bm{\theta}_E);\bm{\theta}_T^{*})$ are functions of the encoder parameters $\bm{\theta}_E$. The objective for optimizing the encoder now reduces to:
\begin{equation}
\label{eq:obj_func_entropy}
\begin{aligned}
& \min_{\bm{\theta}_E} \mbox{ } \mathbb{E}_{\bm{x},t}\left[-\log q_T(t|E(\bm{x};\bm{\theta}_E);\bm{\theta}^*_T)\right] + \log m \\ 
&+\alpha\mathbb{E}_{\bm{x}}\left[\sum_{i=1}^m q_D(s_i|E(\bm{x};\bm{\theta}_E);\bm{\theta}^*_D)\log q_D(s_i|E(\bm{x};\bm{\theta}_E);\bm{\theta}^*_D)\right] \nonumber
\end{aligned}
\end{equation}
\noindent where the first term is minimizing the uncertainty (negative log-likelihood) of the true target attribute label and the second term is maximizing unpredictability (as measured by entropy) across all the classes in the discriminator distribution, thereby, preventing leakage of any information about the sensitive attribute label. In contrast the corresponding objective of the ML-ARL problem is \cite{xie2017controllable},
\begin{equation}
\label{eq:obj_func_minlike}
\begin{aligned}
\min_{\bm{\theta}_E} \mbox{ } & \mathbb{E}_{\bm{x},t}\left[-\log q_T(t|E(\bm{x};\bm{\theta}_E);\bm{\theta}^*_T)\right] \\ 
&+\alpha\mathbb{E}_{\bm{x},s}\left[\log q_D(s|E(\bm{x};\bm{\theta}_E);\bm{\theta}^*_D)\right] \nonumber
\end{aligned}
\end{equation}
\noindent where the first term is minimizing the uncertainty (negative log-likelihood) of the true target attribute label, while the second term is maximizing uncertainty (log-likelihood) of only the true sensitive attribute label. However, by doing so, the encoder inadvertently becomes more certain about the other labels, and can still be informative to an adversary.

\vspace{5pt}
\noindent\textbf{Equilibrium when $s \perp \!\!\! \perp t$: } When the target and sensitive attributes are independent with respect to each other (e.g., age and gender), the two terms in the encoder optimization can both reach their optima simultaneously. Furthermore, the problem reduces to a non-zero sum two player game between the \emph{encoder} and the \emph{discriminator} in the MaxEnt-ARL case and to a zero-sum two player game between the same players in the case of ML-ARL.
\begin{corollary}\label{th2} 
When $s \perp \!\!\! \perp t$, let the optimum discriminator and predictor for an encoder $E$ be $q_D(s|E(\bm{x};\bm{\theta}_E);\bm{\theta}_D^{*})$ and $q_T(t|E(\bm{x};\bm{\theta}_E);\bm{\theta}_T^{*})$ respectively. The optimal encoder $E(\cdot)$ in the MaxEnt-ARL formulation induces a uniform distribution in the discriminator $q_D(s|E(\bm{x};\bm{\theta}_E^{*});\bm{\theta}_D^{*})$ over the classes of the sensitive attribute.
\end{corollary}
\begin{proof} The proof uses the fact that, given a fixed optimal discriminator $D$, $q_T(t|E(\bm{x};\bm{\theta}_E);\bm{\theta}_T^{*})$ is independent of $q_D(s|E(\bm{x};\bm{\theta}_E);\bm{\theta}_D^{*})$ when $s \perp \!\!\! \perp t$. The detailed proof is included in the supplementary material.
\end{proof} 

\vspace{5pt}
\noindent\textbf{Equilibrium when $s \not\perp \!\!\! \perp t$: } When the target and sensitive attributes are related to each other (e.g., beard and gender), the two terms in the encoder optimization cannot reach their optima simultaneously. In both the formulations, ML-ARL and MaxEnt-ARL, the relative optimality of the two objectives depends on the trade-off factor $\alpha$.
\subsection{Convergence Dynamics}
We analyze the standard algorithm (simultaneous stochastic gradient descent) for finding the equilibrium solution of such adversarial games. That is, we take simultaneous gradient steps in $\bm{\theta}_E$, $\bm{\theta}_D$ and $\bm{\theta}_T$, which can be expressed as differential equations of the form:
\begin{equation}
\label{eq:vector-field}
\begin{aligned}
    \dot{\bm{\theta}_D} &= f_D(\bm{\theta}) = \nabla_{\bm{\theta}_D} V_1(\bm{\theta}_E,\bm{\theta}_D) \\
    \dot{\bm{\theta}_T} &= f_T(\bm{\theta}) = \nabla_{\bm{\theta}_T} V_2(\bm{\theta}_E,\bm{\theta}_T)\\
    \dot{\bm{\theta}_E} &= f_E(\bm{\theta}) = \nabla_{\bm{\theta}_E} V_2(\bm{\theta}_E,\bm{\theta}_T) + \alpha V_3(\bm{\theta}_E,\bm{\theta}_T)
\end{aligned}
\end{equation}
\noindent where the gradients $f(\bm{\theta}) = (f_D(\bm{\theta}), f_T(\bm{\theta}), f_E(\bm{\theta}))$ define a \emph{vector field} over $\bm{\theta}=\left(\bm{\theta}_D, \bm{\theta}_T, \bm{\theta}_E\right)$.

The qualitative behavior of the aforementioned non-linear system near any equilibrium point can be determined via \emph{linearization} with respect to that point \cite{khalil1996nonlinear}. Restricting our attention to a sufficiently small neighborhood of the equilibrium point, the non-linear state equations in (\ref{eq:vector-field}) can be approximated by a linear state equation:
\begin{equation}
\dot{\bm{\theta}} = \bm{J}\bm{\theta}
\label{eq:jacobian}
\end{equation}
\noindent where,
$\left.\bm{J}=\begin{bmatrix} \frac{\partial f_D(\bm{\theta})}{\partial \bm{\theta}_D} & \frac{\partial f_D(\bm{\theta})}{\partial \bm{\theta}_T} & \frac{\partial f_D(\bm{\theta})}{\partial \bm{\theta}_E} \\
\frac{\partial f_T(\bm{\theta})}{\partial \bm{\theta}_D} & \frac{\partial f_T(\bm{\theta})}{\partial \bm{\theta}_T} & \frac{\partial f_T(\bm{\theta})}{\partial \bm{\theta}_E} \\
\frac{\partial f_E(\bm{\theta})}{\partial \bm{\theta}_D} & \frac{\partial f_E(\bm{\theta})}{\partial \bm{\theta}_T} & \frac{\partial f_E(\bm{\theta})}{\partial \bm{\theta}_E}
\end{bmatrix}\right\rvert_{\bm{\theta}=\bm{\theta}^*}$
is the Jacobian of the vector field evaluated at the chosen equilibrium point $\bm{\theta}^*=(\bm{\theta}^*_D,\bm{\theta}^*_T,\bm{\theta}^*_E)$. For small neighborhoods around an equilibrium, the trajectories of the non-linear system in (\ref{eq:vector-field}) is expected to be ``close" to the trajectories of the linear approximate system in (\ref{eq:jacobian}).
\begin{theorem}[Linearization]
    Let $\bm{x}=\bm{0}$ be an equilibrium point for the non-linear system, $\dot{\bm{x}}=f(\bm{x)}$, where 
    $f:\mathcal{D} \rightarrow \mathbb{R}^n$ is continuously differentiable and $\mathcal{D}$ is a neighborhood of the origin. Let, $\left.\bm{J}=\frac{\partial f}{\partial \bm{x}}\right\rvert_{\bm{x}=\bm{0}}$. Then,
    \begin{itemize}
        \item The origin is asymptotically stable if Re$(\lambda_i) < 0$ for all eigenvalues of $\bm{J}$.
        \item The origin is unstable if Re$(\lambda_i)\geq 0$ for one or more of the eigenvalues of $\bm{J}$.
    \end{itemize}
\end{theorem}
\begin{proof}
    See Theorem 4.7 of \cite{khalil1996nonlinear}.
\end{proof}

\section{Numerical Experiments \label{sec:experiments}}
In this section we will evaluate the efficacy of the proposed \emph{Maximum Entropy Adversarial Representation Learning} model and compare it with other \emph{Adversarial Representation Learning} baselines.
\subsection{Three Player Game: Linear Case}
\begin{figure}
    \centering
    \includegraphics[width=0.5\textwidth]{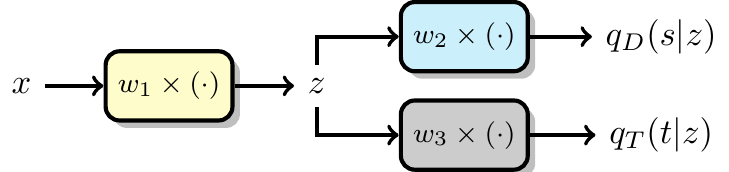}
    \caption{Three Player Game: Linear Example \label{fig:linear-three-player-game}}
\end{figure}
As an illustrative example we analyze the convergence of both ML-ARL and MaxEnt-ARL under the same setting. The encoder, discriminator and predictor are linear models with multiplicative weights $w_1,w_2$ and $w_3$, respectively. We limit our model to this three variable setting for ease of analysis and visualization. Both predictor and the discriminator are optimizing cross-entropy loss on binary $\{0,1\}$ labels. To observe the game between the three players we provide same data sample $x=1$ yet with different target and sensitive labels i.e., 4 samples with $\{00, 01, 10, 11\}$ for target and sensitive labels. Loss is calculated as the average over all samples and corresponding vector field values are also computed. The stationary point of this game, for both ML-ARL and MaxEnt-ARL, is at  $(w_1=0,w_2=0,w_3=0)$ and the gradient of the loss functions are zero at this point. We consider a small ($30 \times 30 \times 30$ grid) neighborhood around the stationary point in the range $[-0.01, 0.01]$ for weights $w_1,w_2,w_3$ and visualize trajectories by following the vector field of the game.

Figure~\ref{fig:trajectory_1d_gaussian} shows streamline plots of the vector field around $(0,0,0)$ for a point starting at the green location. In the ML-ARL case, we observe that when the predictor is fixed at $w_3=0$, the trajectory for the encoder and the discriminator does not converge and rotates around the stationary point. In contrast, for the MaxEnt-ARL method converges to the stationary point. When $w_1=0$, the streamlines for both ML-ARL and MaxEnt-ARL converge to $(0,0)$. For an alternate formulation, where the discriminator is of the form $D=z^2+b_2$, we found convergent behavior for both ML-ARL and MaxEnt-ARL.

\begin{figure*}[!ht]
    \centering
    \begin{subfigure}[t]{0.24\textwidth}
        \includegraphics[width=\textwidth]{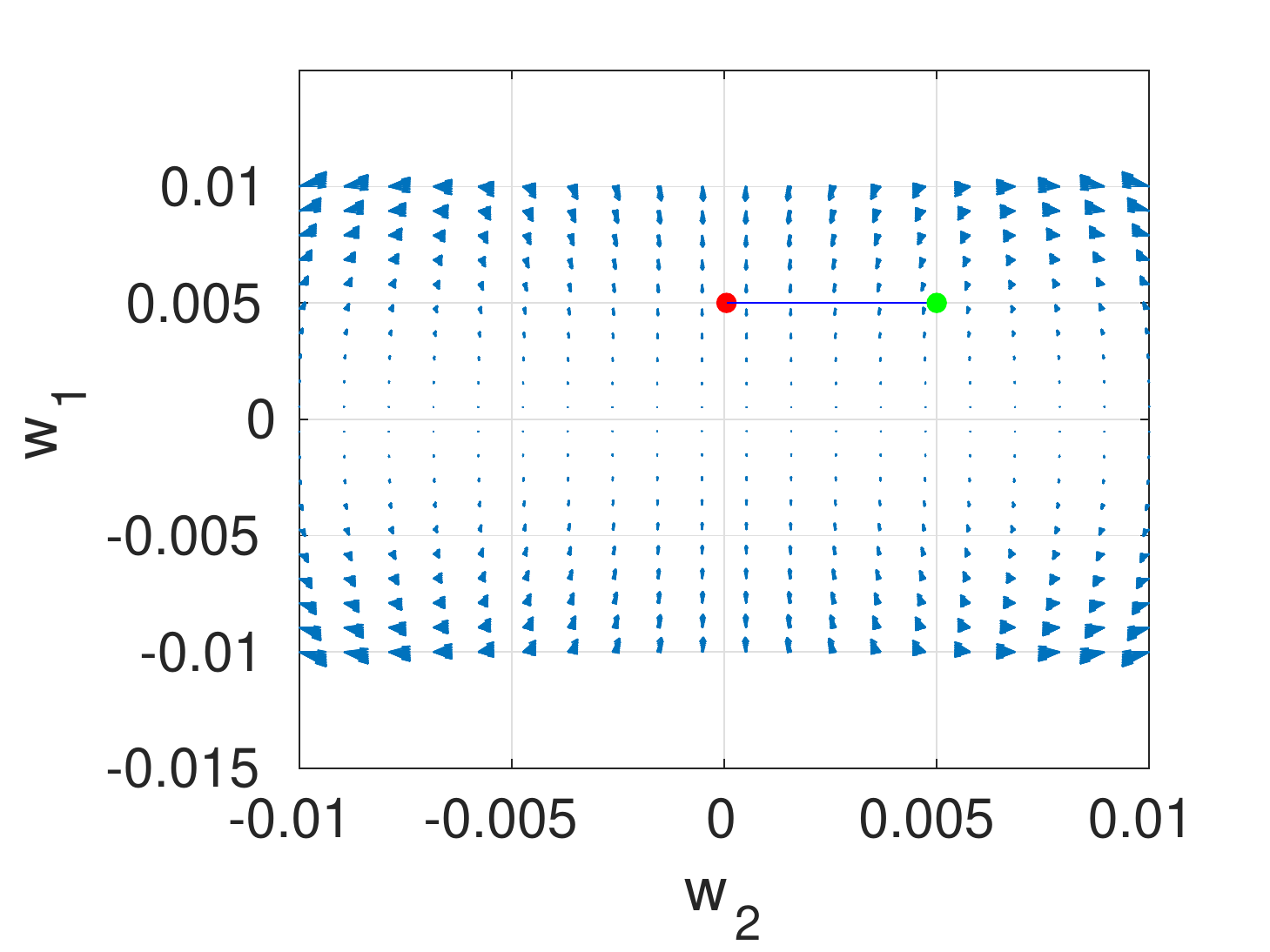}
        \caption{ML-ARL: Trajectory}
    \end{subfigure}
    \begin{subfigure}[t]{0.24\textwidth}
        \includegraphics[width=\textwidth]{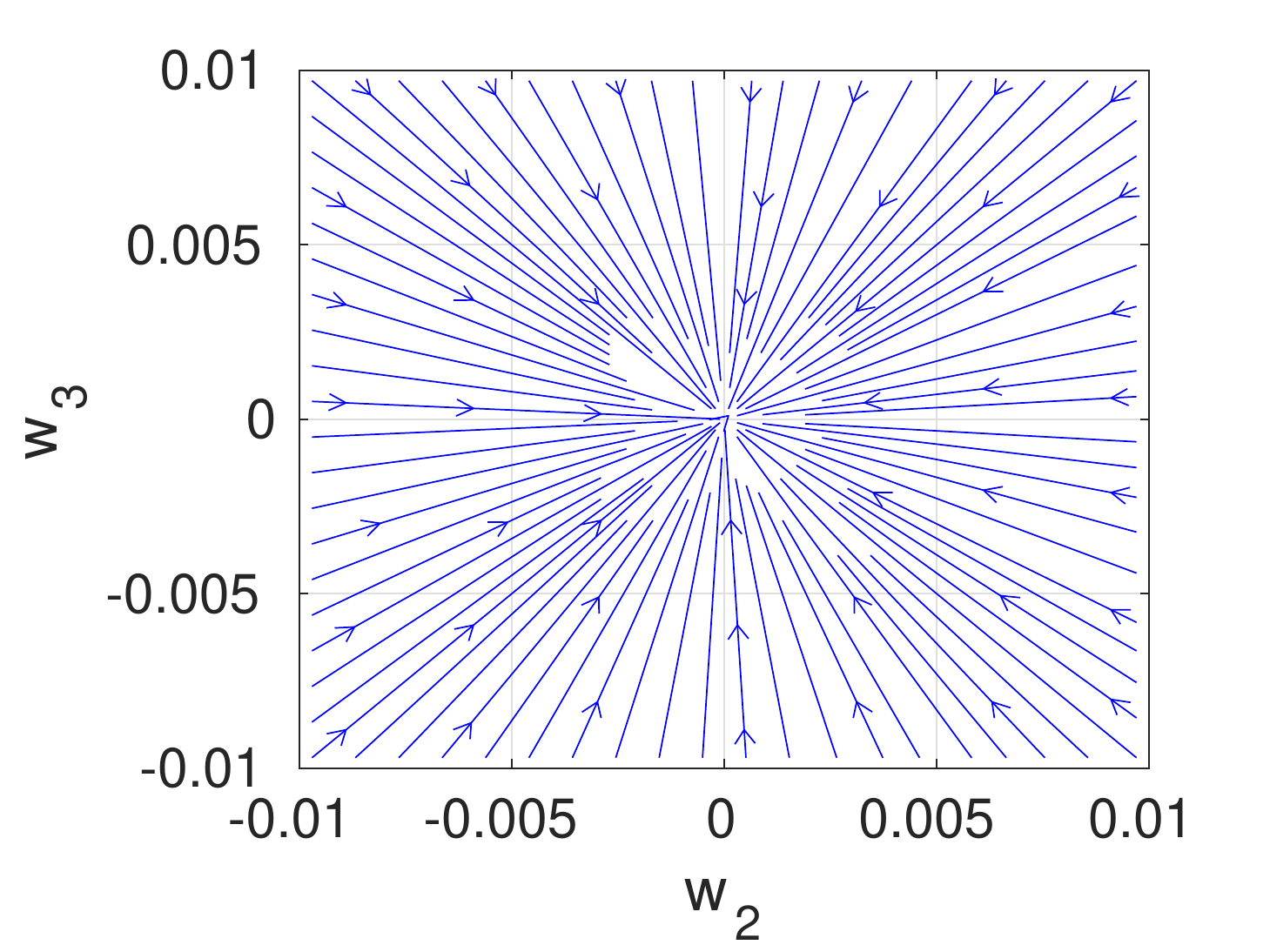}
        \caption{ML-ARL: $w_1=0$}
    \end{subfigure}
     \begin{subfigure}[t]{0.24\textwidth}
        \includegraphics[width=\textwidth]{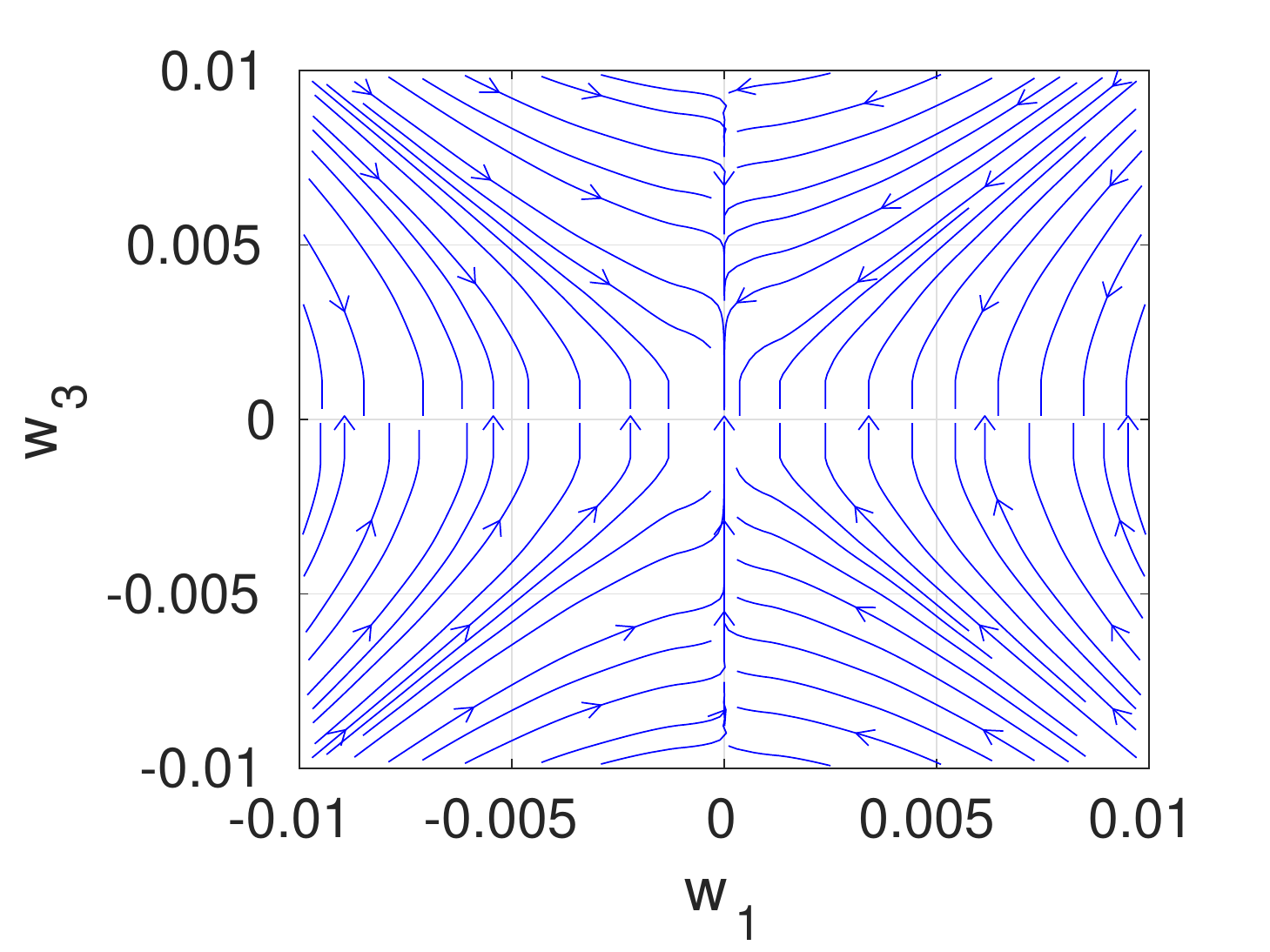}
        \caption{ML-ARL: $w_2=0$}
    \end{subfigure}
    \begin{subfigure}[t]{0.24\textwidth}
        \includegraphics[width=\textwidth]{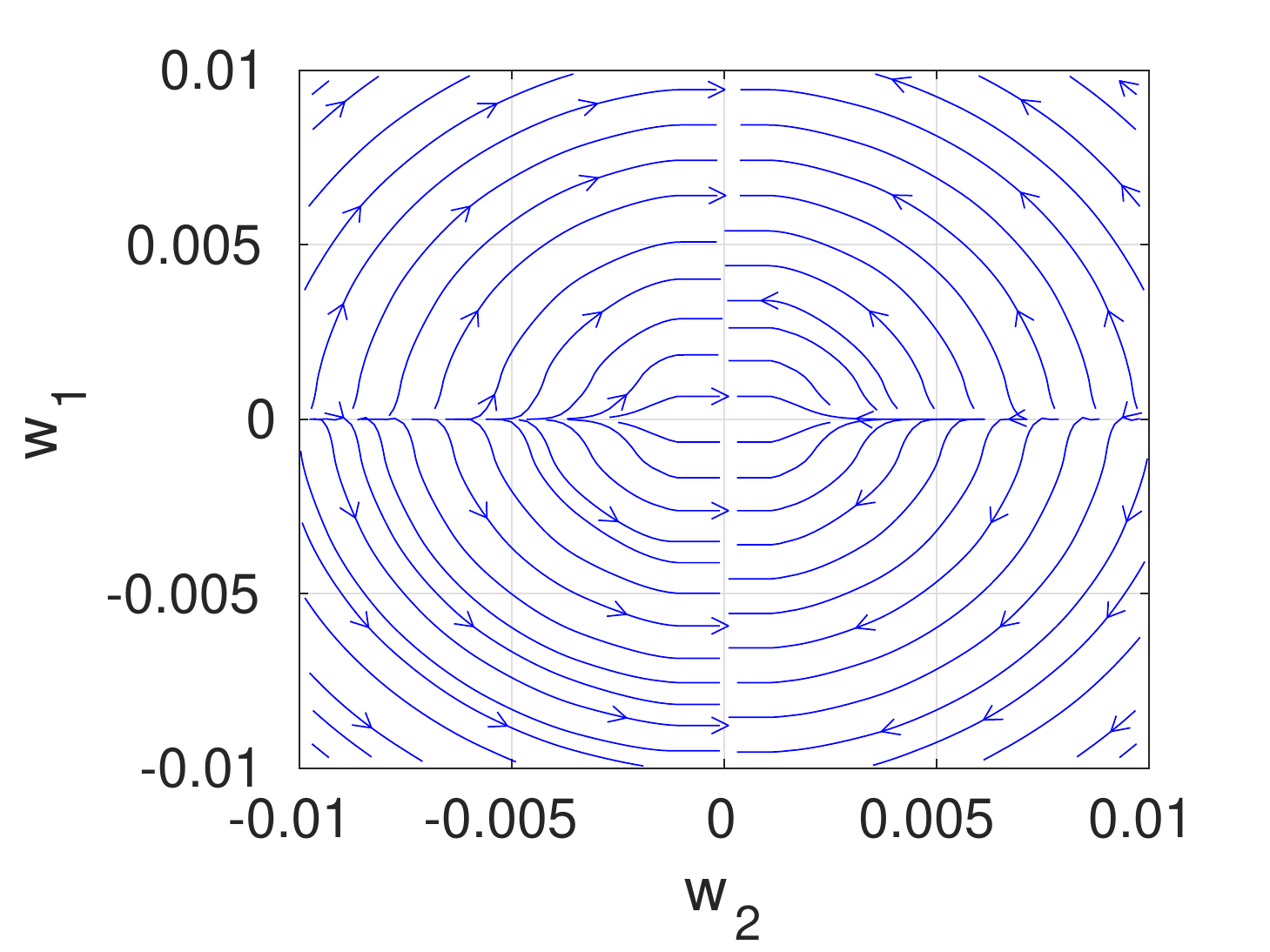}
        \caption{ML-ARL: $w_3=0$}
    \end{subfigure}
    \begin{subfigure}[t]{0.24\textwidth}
        \includegraphics[width=\textwidth]{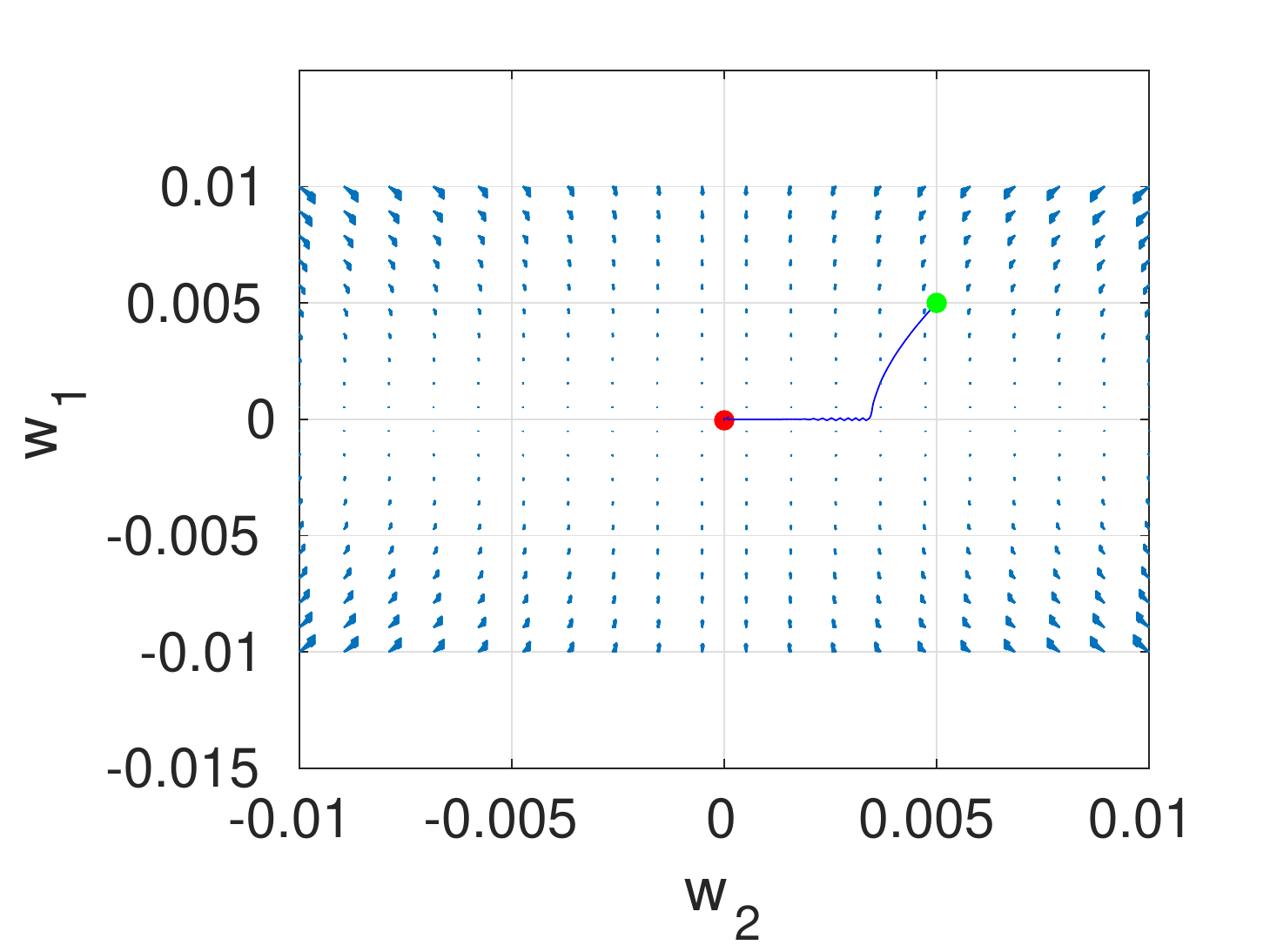}
        \caption{MaxEnt-ARL: Trajectory}
    \end{subfigure}
    \begin{subfigure}[t]{0.24\textwidth}
        \includegraphics[width=\textwidth]{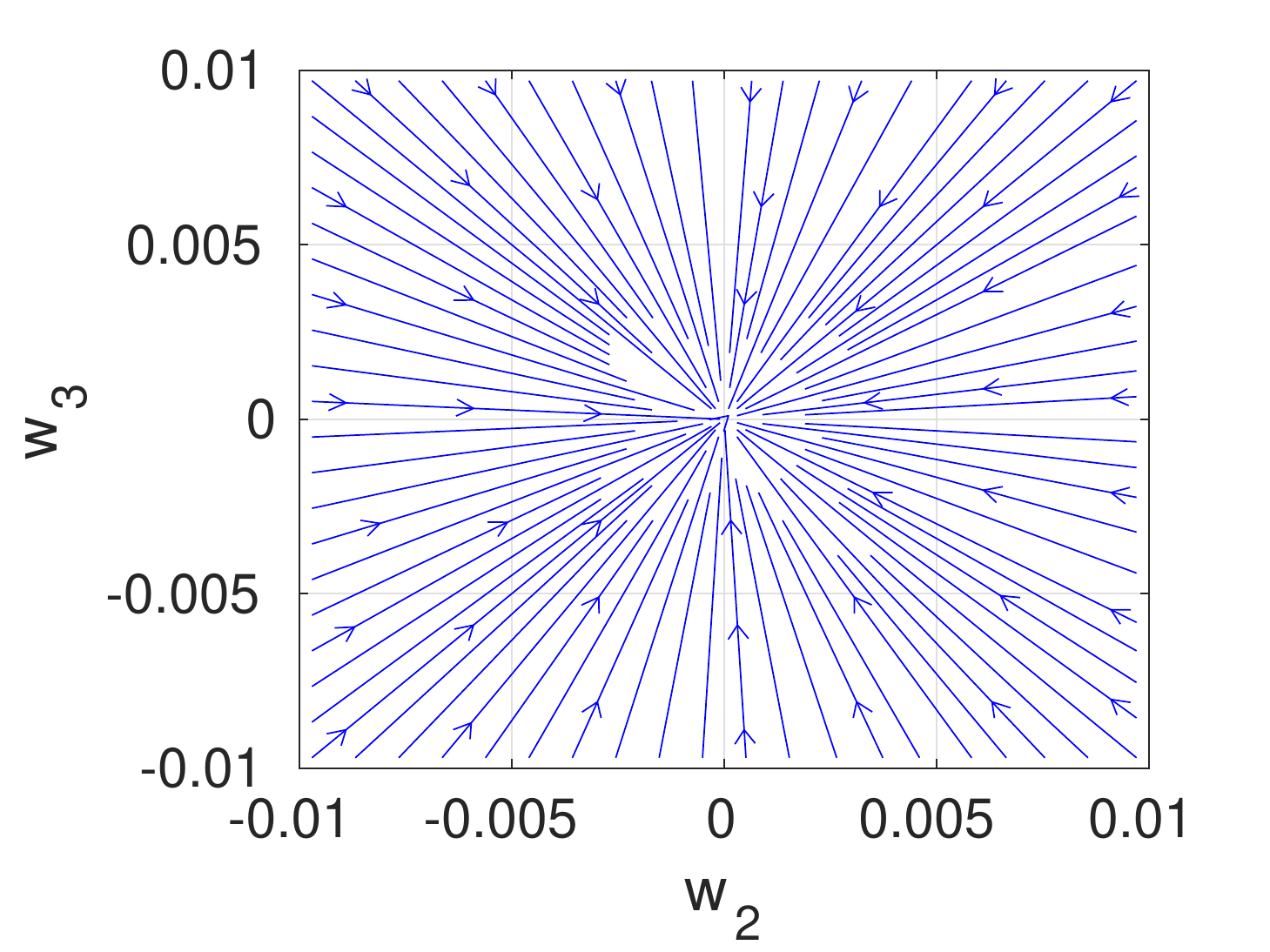}
        \caption{MaxEnt-ARL:  $w_1=0$}
    \end{subfigure}
     \begin{subfigure}[t]{0.24\textwidth}
        \includegraphics[width=\textwidth]{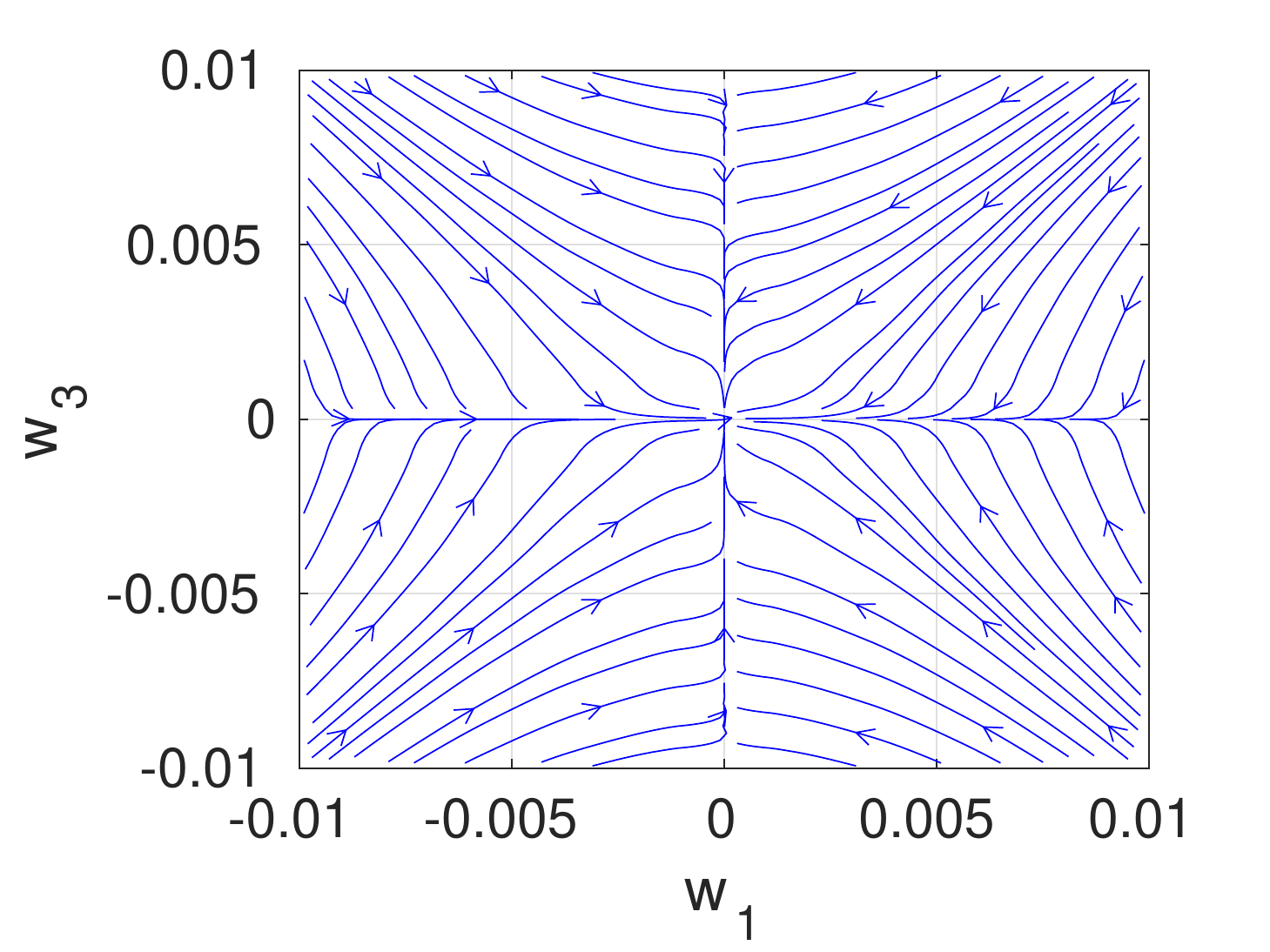}
        \caption{MaxEnt-ARL: $w_2=0$}
    \end{subfigure}
    \begin{subfigure}[t]{0.24\textwidth}
        \includegraphics[width=\textwidth]{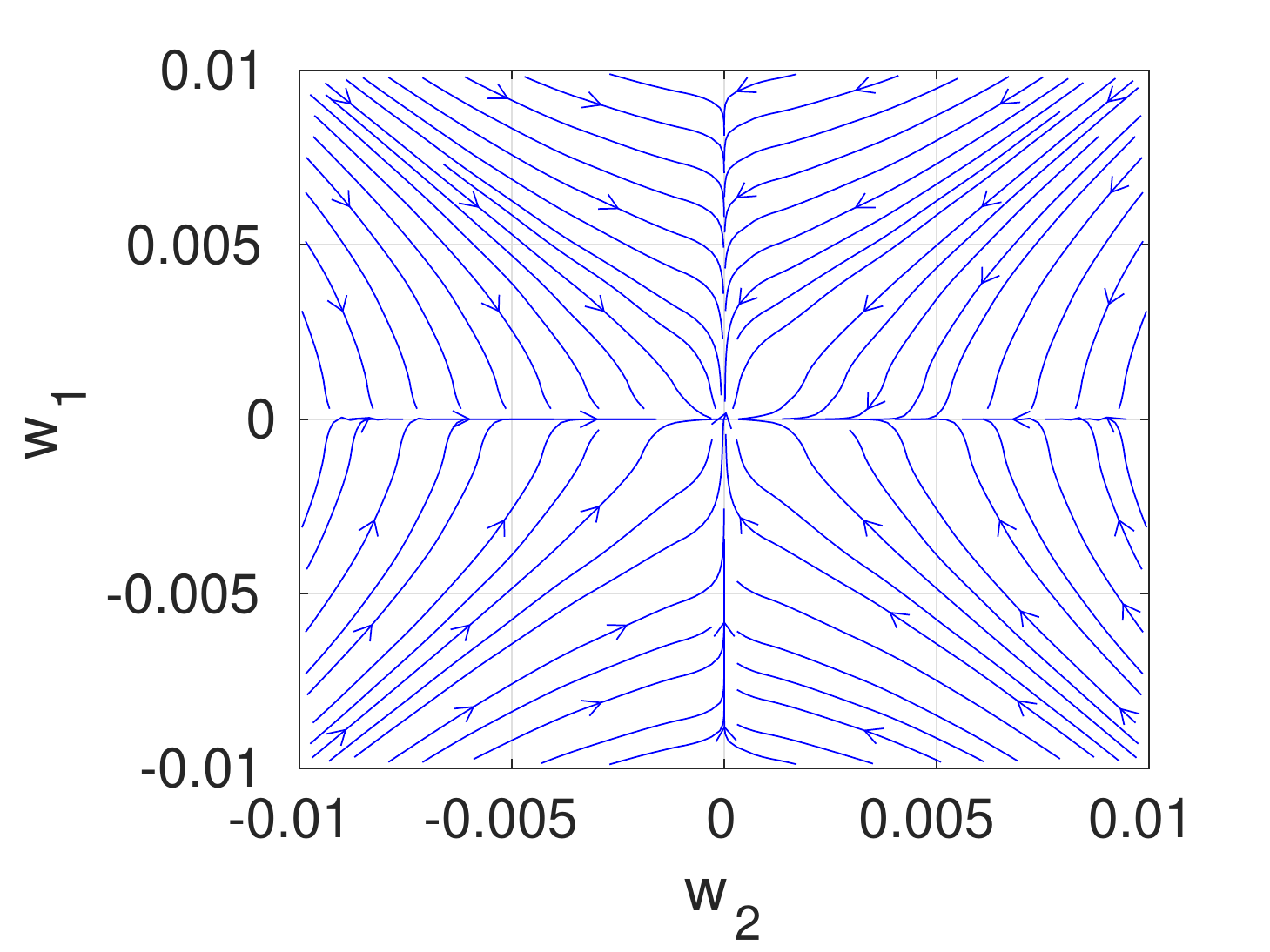}
        \caption{MaxEnt-ARL: $w_3=0$}
    \end{subfigure}
    \caption{Streamline plots for linear three-player game near stationary point (0,0,0). Trajectories start at the {\color{green}{green}} point and converge to the {\color{red}{red}} point by following the vector field. (a) and (e) shows the top-view of the 3-D trajectories. 
    When $w_1=0$ the trajectories suggest that both ML-ARL and MaxEnt-ARL converge to the local optima, ($w_1=w_2=w_3=0$). When $w_2=0$, the MaxEnt-ARL trajectories converge to the local optima. The ML-ARL trajectories converge to the optima only when they start far away from $0$ along $w_3$. The trajectories starting closer to $w_3=0$, however, do not converge to $w_1=0$. When $w_3=0$, the game reduces to a two-player adversarial game (akin to a GAN\cite{goodfellow2014generative}), where ML-ARL shows non-convergent cyclic behavior while MaxEnt-ARL converges.\label{fig:trajectory_1d_gaussian}}
\end{figure*}

\subsection{Mixture of Gaussians}
In this experiment we seek to visualize and compare the representation learned by MaxEnt-ARL and ML-ARL. We consider a mixture of 4 Gaussians with means $\mu$ at  $((1,1),(2,1.5),(1.5,2.5),(2.5,3))$ and variance $\sigma=0.3$ in each case. Our model is a neural network with 2 hidden layer with 2 neuron in each layer. Each data sample has two attributes, color and shape. We setup the ARL problem with shape as the target attribute and color as the sensitive attribute. The encoder is a neural network with one hidden layer, mapping the 2-D shape into another 2-D embedding, and both the predictor and discriminator are logistic regression classifiers. The trade-off parameter is set to $\alpha=0.1$ and the parameters are learned using the Adam optimizer with learning rate of $10^{-4}$. After learning the embedding function, we freeze its parameters and learn a logistic classifier as the adversary. The test accuracy of the adversary is 63\% for MaxEnt-ARL and 70\% for ML-ARL. Therefore, by optimizing the entropy instead of the likelihood MaxEnt-ARL is able to leak less information about the sensitive label compared to ML-ARL. Figure \ref{fig:toy_4_gaussian} shows the data and the learned embeddings.
\begin{figure*}[!ht]
    \centering
    \begin{subfigure}[t]{0.32\textwidth}
        \includegraphics[width=\textwidth]{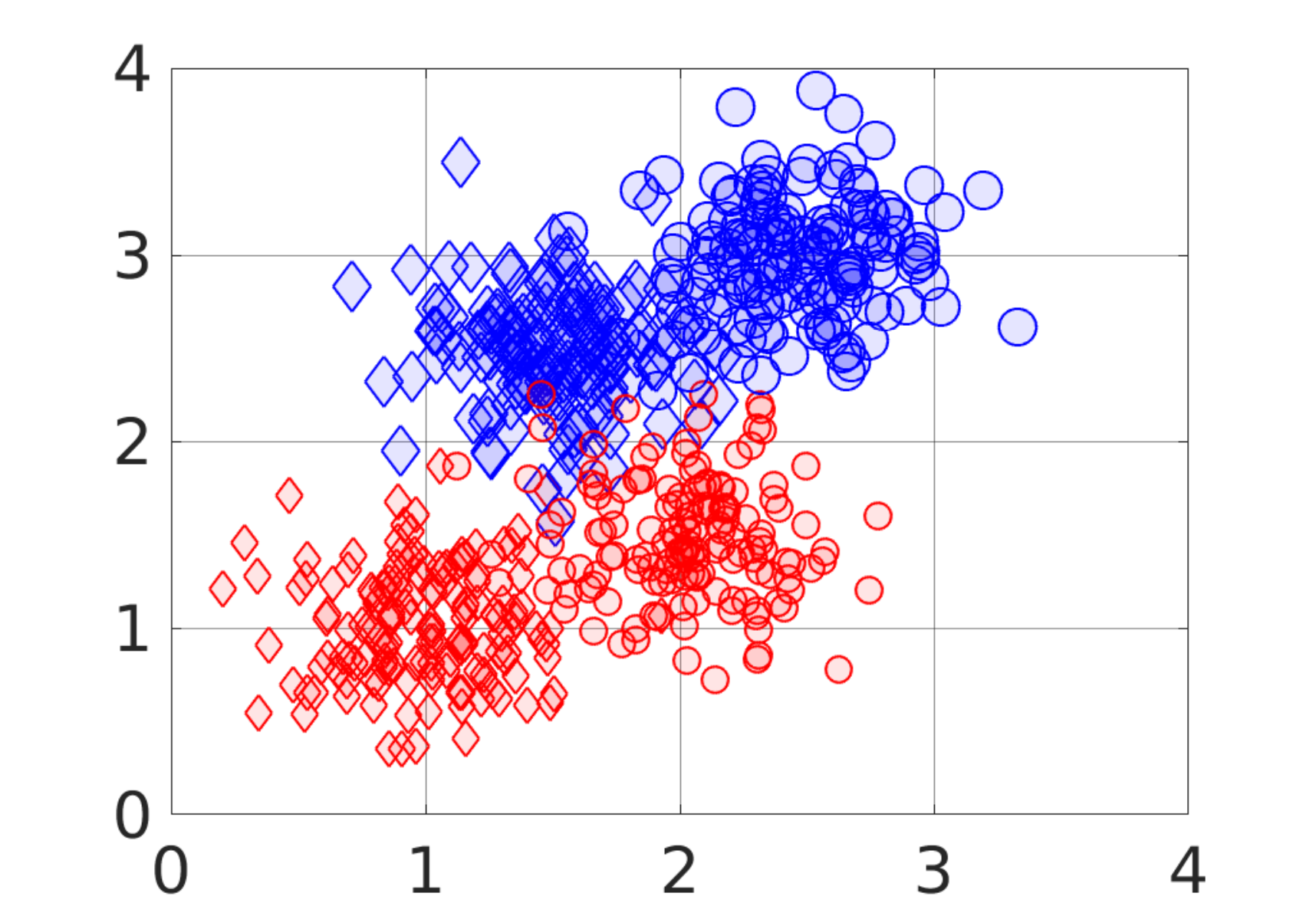}
        \caption{Input}
    \end{subfigure}
    \begin{subfigure}[t]{0.32\textwidth}
        \includegraphics[width=\textwidth]{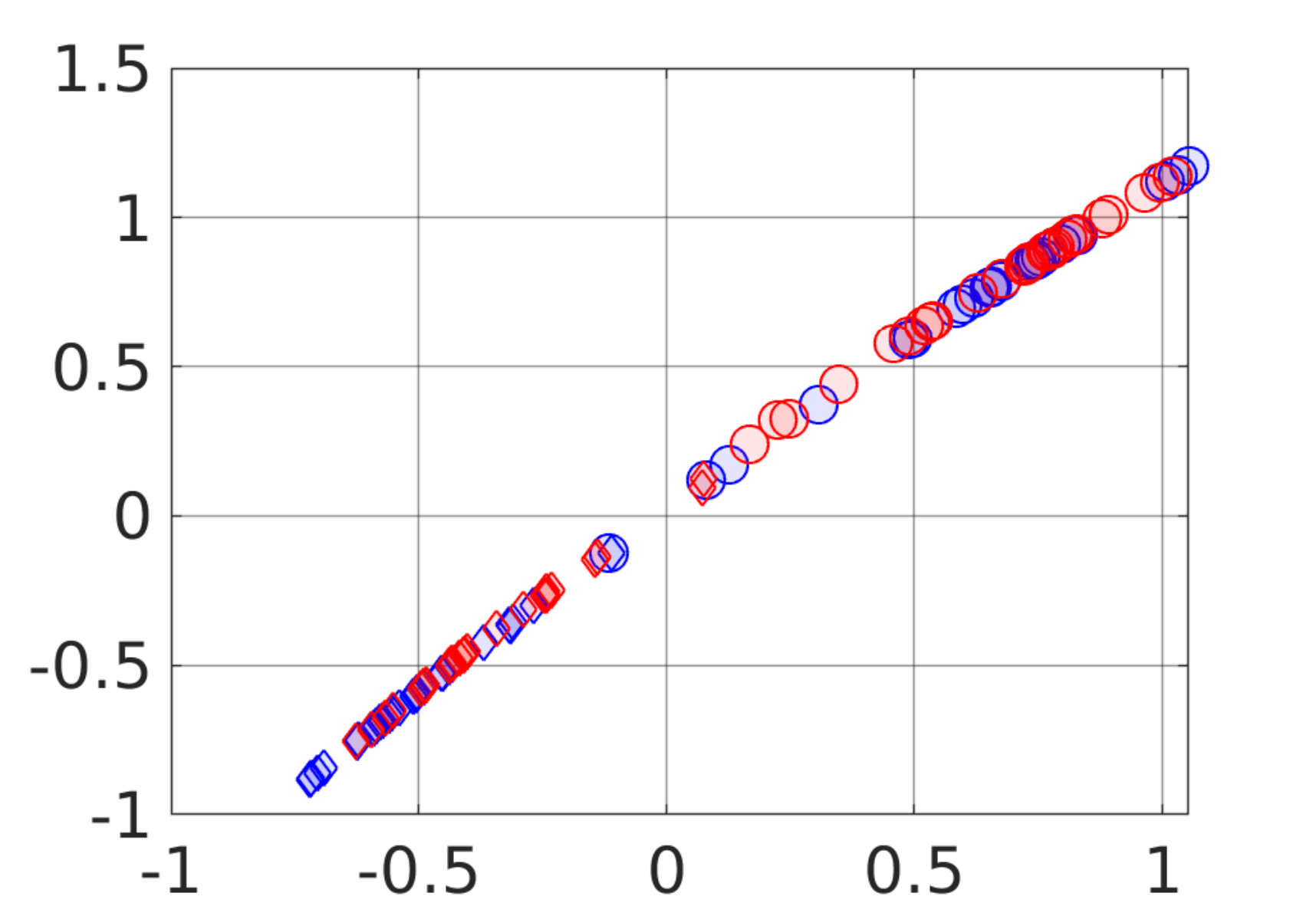}
        \caption{ML-ARL}
    \end{subfigure}
     \begin{subfigure}[t]{0.32\textwidth}
        \includegraphics[width=\textwidth]{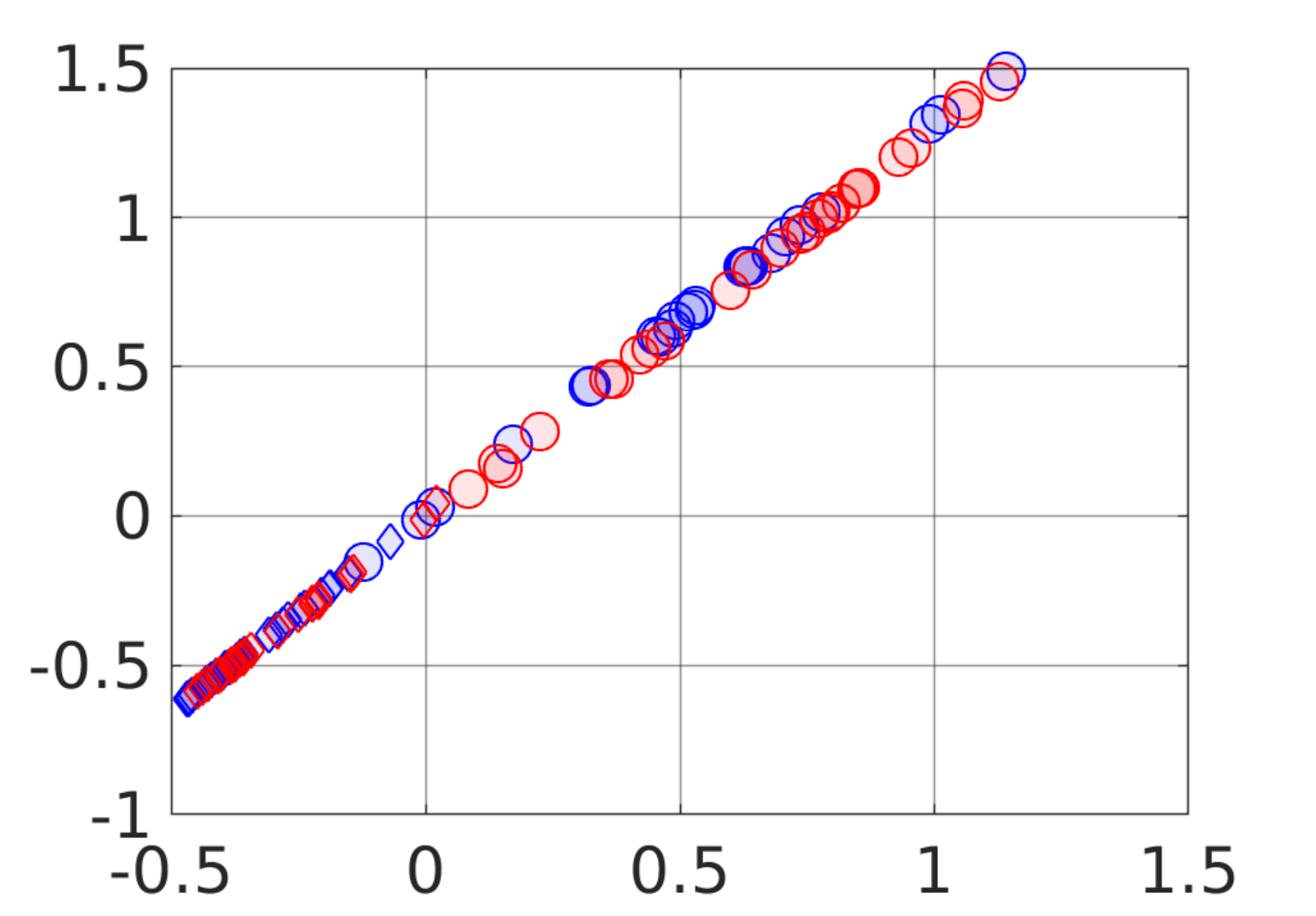}
        \caption{MaxEnt-ARL}
    \end{subfigure}
    \caption{Samples from four Gaussians with target (shape) and sensitive attributes (color). (a) input space, (b) learned embedding $z$ for ML-ARL, (c) learned embedding for MaxEnt-ARL. We can now notice that ML-ARL has some isolated samples with different colors (sensitive label), while MaxEnt-ARL results in slightly better mixing of the colors.\label{fig:toy_4_gaussian}}
\end{figure*}

\begin{figure}[h]
    \centering
    \begin{subfigure}[t]{0.23\textwidth}
        \includegraphics[width=\textwidth]{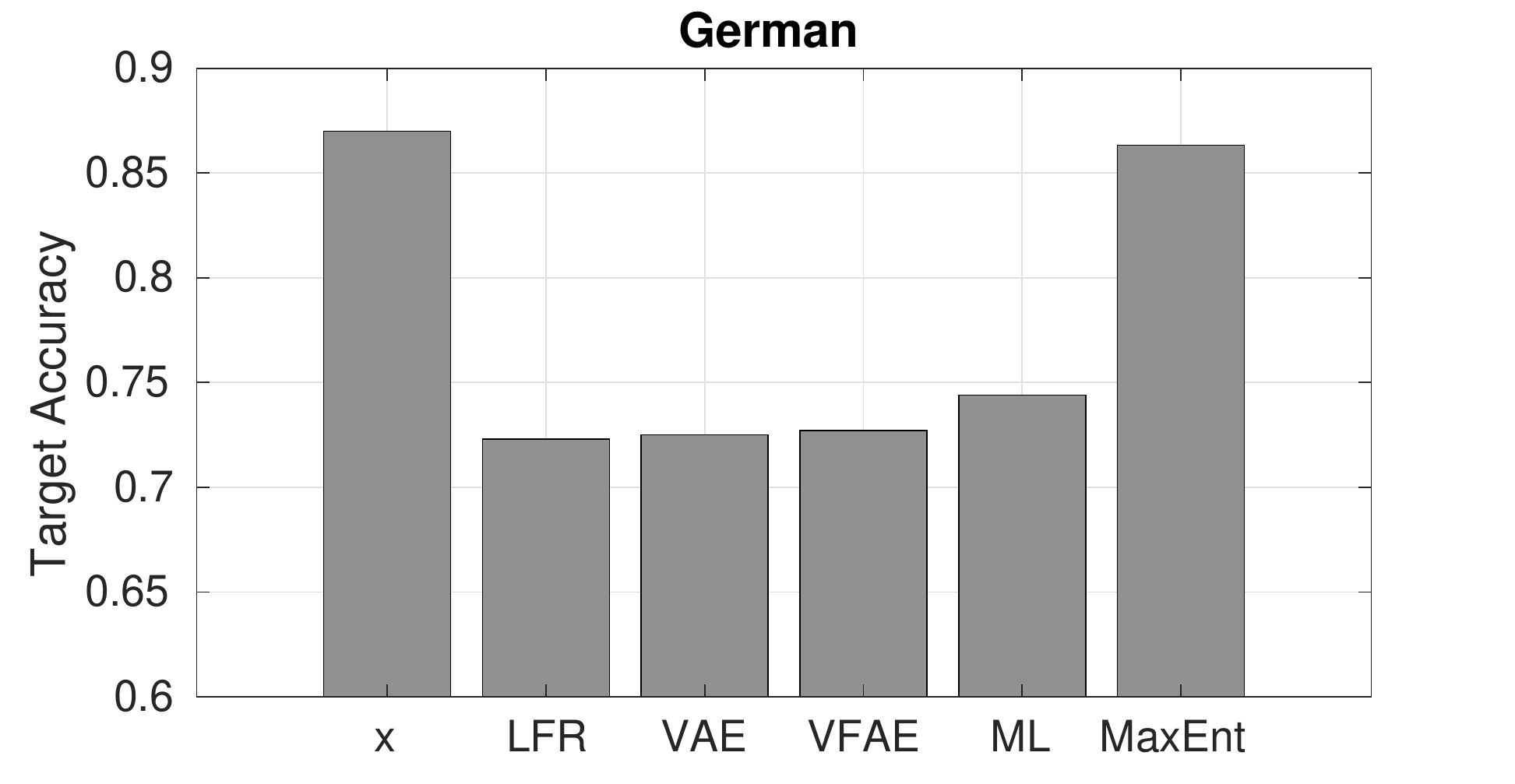}
        \caption{Target Attribute: Credit}
    \end{subfigure}
    \begin{subfigure}[t]{0.23\textwidth}
        \includegraphics[width=\textwidth]{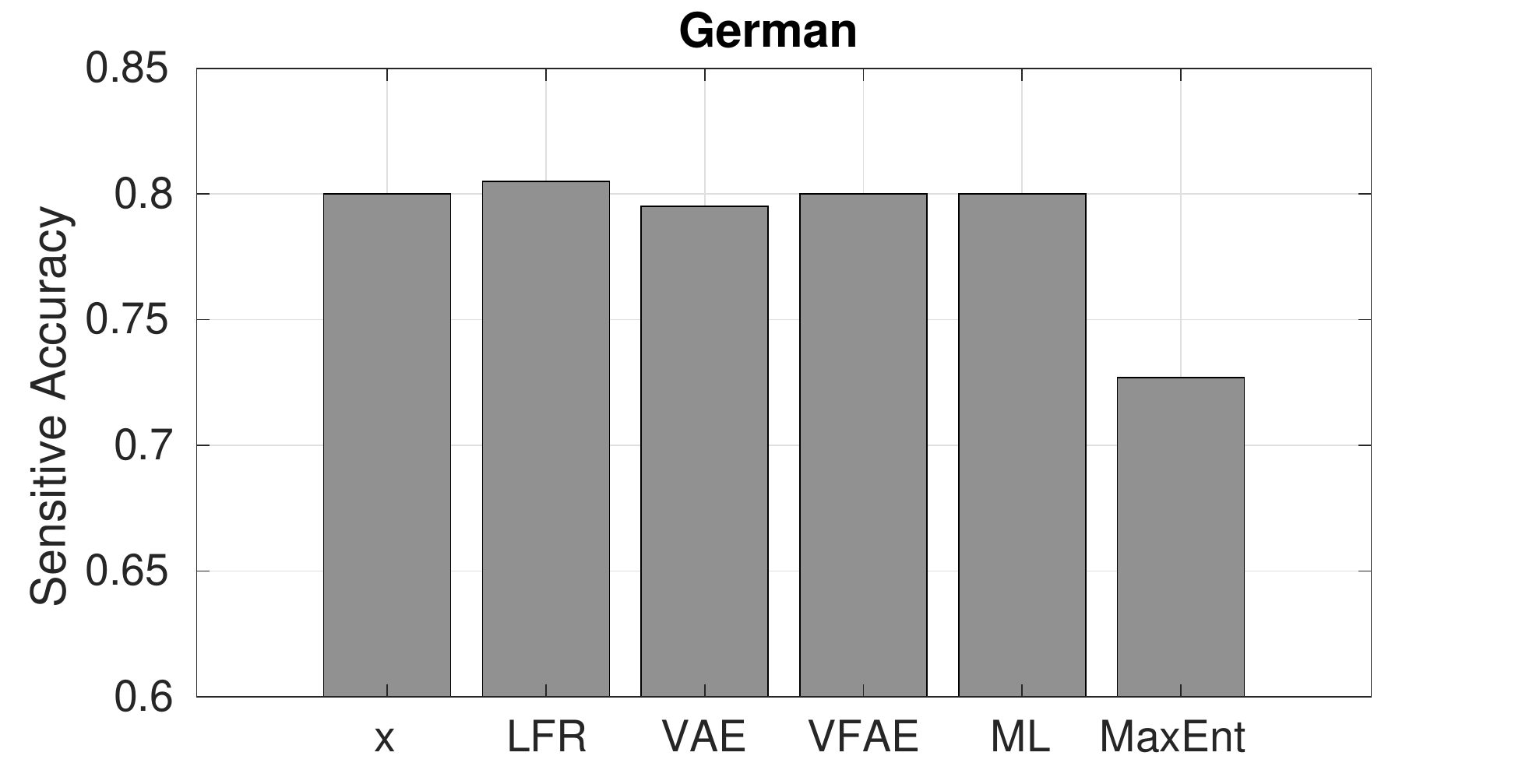}
        \caption{Sensitive Attribute: Gender}
    \end{subfigure}
    \begin{subfigure}[t]{0.23\textwidth}
        \includegraphics[width=\textwidth]{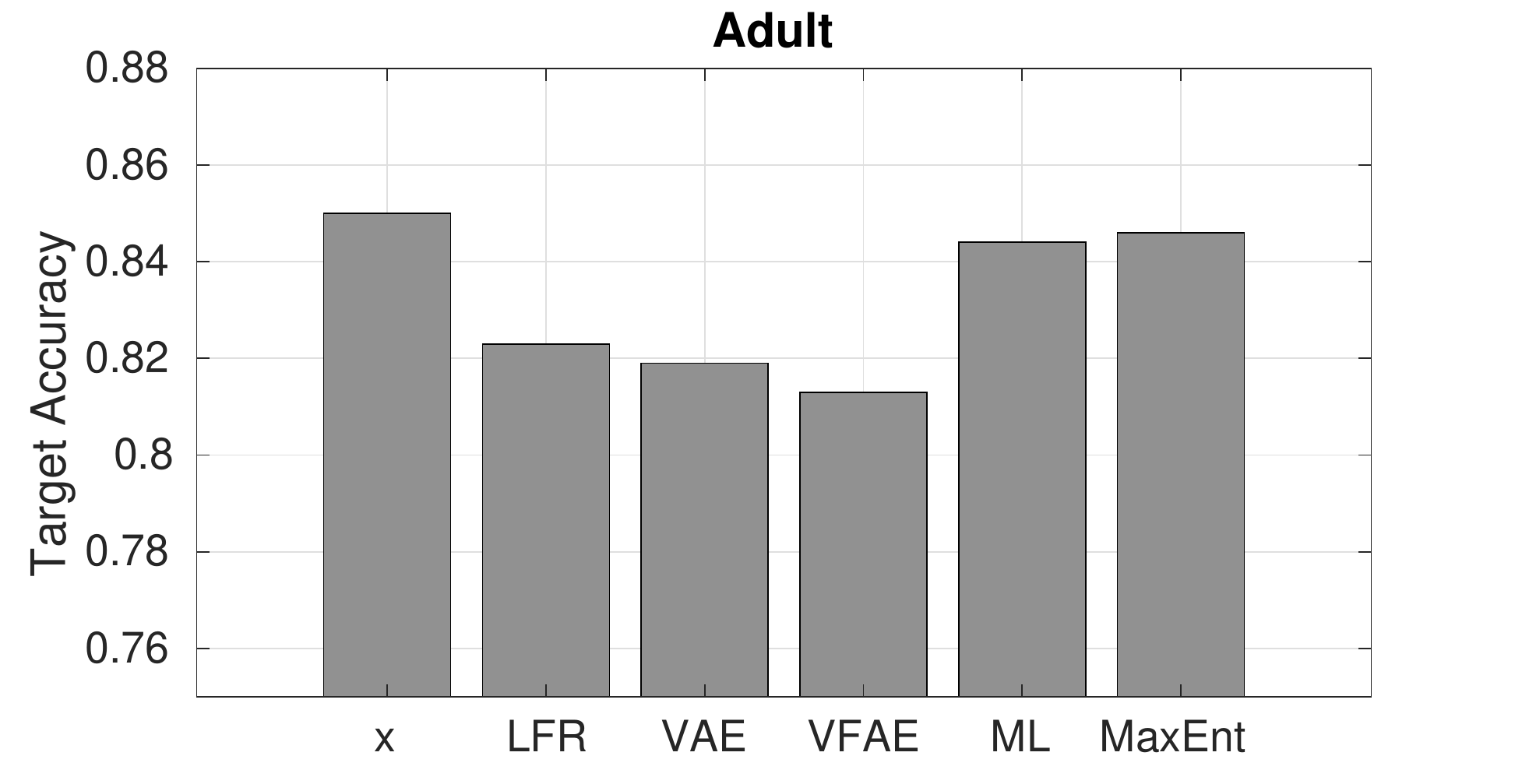}
        \caption{Target Attribute: Income}
    \end{subfigure}
    \begin{subfigure}[t]{0.23\textwidth}
        \includegraphics[width=\textwidth]{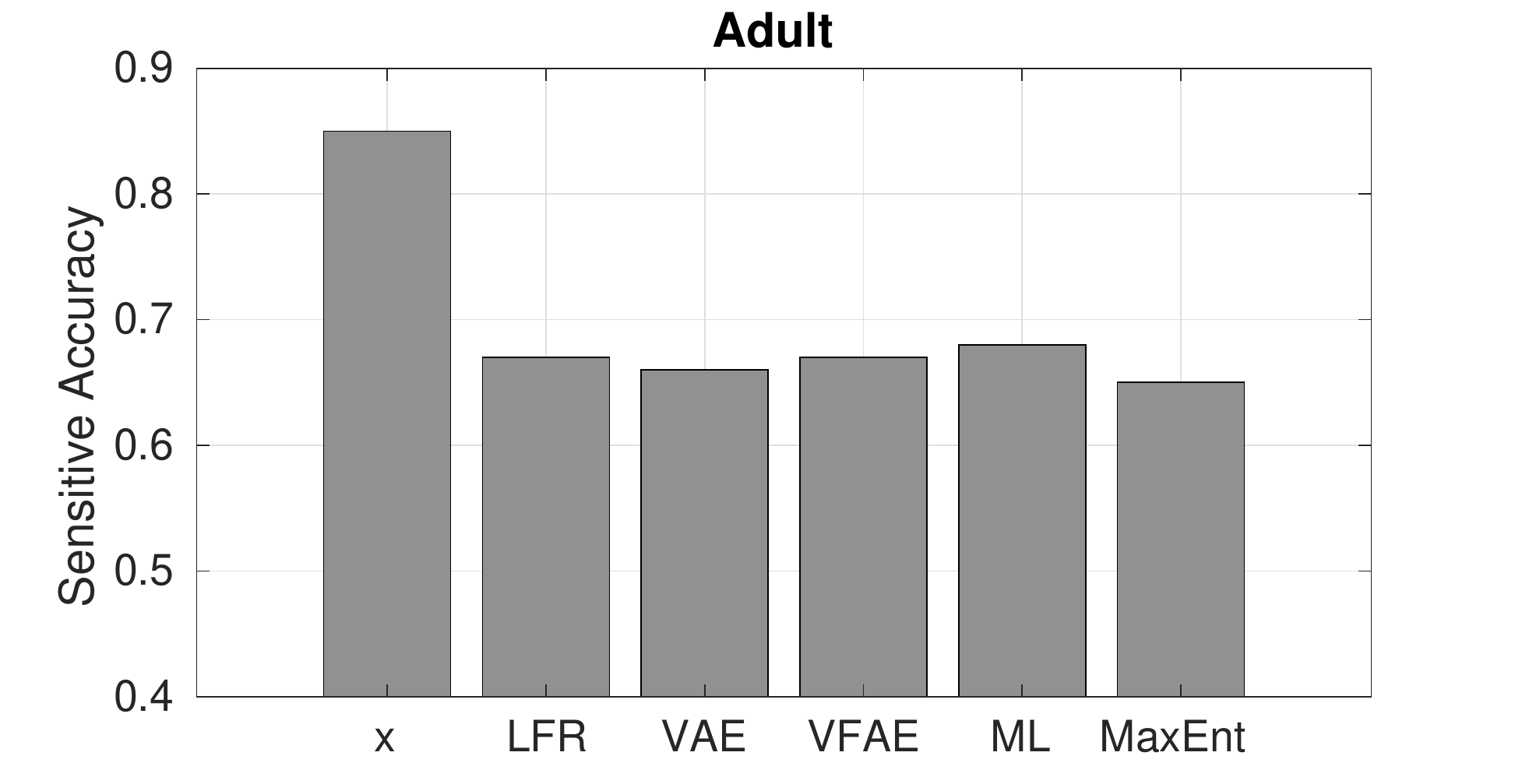}
        \caption{Sensitive Attribute: Gender}
    \end{subfigure}
    \caption{Representation Learning for Fair Classification \label{fig:fair}}
\end{figure}

\subsection{Fair Classification}
We consider the setting of fair classification on two datasets from the UCI ML-repository~\cite{Dua:2019}, (a) The German credit dataset
with 20 attributes for 1000 instances with target label being classifying bank account holders with good or bad credit and gender being the sensitive attribute, (b) The Adult income dataset has 45,222 instances with 14 attributes. The target is a binary label of annual income more or less than $\$50,000$, while gender is the sensitive attribute. For both ML-ARL and MaxEnt-ARL, the encoder is a NN with one hidden layer, discriminator is a NN with 2 hidden layers, and target predictor is linear logistic regression. Following ML-ARL~\cite{xie2017controllable} we choose 64 units in each hidden layer. We compare both ARL formulations with state-of-the-art baselines LFR (Learning Fair Representations~\cite{zemel2013learning}), VAE (Variational Auto-encoder~\cite{kingma2013auto}) and VFAE (Variational Fair Auto-encoder~\cite{louizos2015variational}). For MaxEnt-ARL, after learning the embedding, we again learn an adversary to extract the sensitive attribute.

Figure~\ref{fig:fair} show the results for the German and Adult datasets, for both the target and sensitive attributes. For German data, MaxEnt-ARL's prediction accuracy is 86.33\% which is close to that of the original data (87\%). Other models such as, LFR, VAE, VFAE and ML-ARL have target accuracies of 72.3\%, 72.5\%, 72.7\% and 74.4\% respectively. On the other hand, for the sensitive attribute, the MaxEnt-ARL adversary's accuracy is 72.7\%. Other models reveal much more information with adversary accuracies of 80\%, 80.5\%, 79.5\%, 79.7\% and 80.2\% for the original data, LFR, VAE, VFAE and ML-ARL, respectively. For the adult income dataset, the target accuracy for original data, ML-ARL and MaxEnt-ARL is 85\%, 84.4\% and 84.6\%, respectively, while the adversary's performance on the sensitive attribute is 67.7\% and 65.5\% for ML-ARL and MaxEnt-ARL, respectively.

\subsection{Illumination Invariant Face Classification}
We consider the task of face classification under different illumination conditions. We used the Extended Yale B dataset~\cite{georghiades2001few} comprising of face images of 38 people under different lighting conditions (directions of the light source) : upper right, lower right, lower left, upper left, or the front. Our target task is to identify one of the 38 people in the dataset with the direction of the light source being the sensitive attribute. We follow the experimental setup of Xie et al.~\cite{xie2017controllable} and Louizos et al.~\cite{louizos2015variational} using the same train/test split strategy and no validation set. $38\times5 = 190$  samples are used for training and the rest of the 1,096 data samples are used for testing. Following the model setup in ~\cite{xie2017controllable}, the encoder is a one layer neural network, target predictor is a linear layer and the discriminator has two hidden layers where each hidden layer consists of 100 units. The parameters are trained using Adam~\cite{kingma2014adam} with a learning rate of $10^{-4}$ and weight decay of $5\times 10^{-2}$.

\begin{table}[t]
\caption{Illumination Invariant Face Classification (\%)\label{tab:yaleface_b}}
\begin{center}
 \begin{tabular}{c ||c |c} 
 \hline
 Method & $s$ (lighting) & $t$ (identity) \\ \hline\hline
 LR & 96 & 78\\ \hline
NN + MMD~\cite{li2014learning} &- & 82 \\ \hline
VFAE~\cite{louizos2015variational} & 57 & 85 \\ \hline
ML-ARL~\cite{xie2017controllable} & 57 & \textbf{89}\\ \hline
Maxent-ARL & \textbf{40 }& \textbf{89} \\ \hline \hline
\end{tabular}
\end{center}
\end{table}

We report baseline ~\cite{li2014learning, louizos2015variational, xie2017controllable} results for this experiment in Table~\ref{tab:yaleface_b} and compare with the proposed MaxEnt-ARL framework. Louizos et al.~\cite{louizos2015variational} regularize their neural networks via Maximum Mean Discrepancy to remove lighting conditions from data whereas Xie et al.~\cite{xie2017controllable} use the ML-ARL framework. The MaxEnt-ARL achieves an accuracy of 89\% for identity classification (same as ML-ARL) while outperforming MMD (82\%) and VFAE (85\%). In terms of protecting sensitive attribute i..e, illumination direction, adversary's classification accuracy reduces from 57\% for ML-ARL to 40.2\% for MaxEnt-ARL. It is clear from the table that, MaxEnt-ARL is able to remove more information from the image compared to the baselines.

\subsection{CIFAR-10}
We create a new binary target classification problem on the CIFAR-10 dataset\cite{krizhevsky2009learning}. The CIFAR-10 dataset consists of 10 basic classes, namely, (`airplane', `automobile', `bird', `cat', `deer', `dog', `frog', `horse', `ship', `truck'). We divide the classes into two groups: living and non-living objects. We expect the living objects to have visually discriminative properties like smooth shapes compared to regular geometric shapes of non-living objects. The target task is binary classification of an image into these two supersets with the underlying class label being the sensitive attribute. For example, the task of classifying an object as living (`dog' or `cat') or non-living (`ship' or `truck') should not reveal any information about its underlying identity (`dog', `cat', `truck' or `ship'). But as we will see, this is a challenging problem and the image representation might not be able to prevent leakage of the sensitive label.

\vspace{5pt}
\noindent\textbf{Implementation Details:} We adopt the ResNet-18 \cite{he2016identity} architecture as the encoder, and the discriminator and adversary are 2-layered neural networks with 256 and 64 neurons, respectively. The encoder and the target predictor are trained using SGD with momentum of 0.9, learning rate of $10^{-3}$ and weight-decay of $10^{-3}$ for the prediction task. Both the discriminator and the adversary, however, are trained using Adam with a learning rate of $10^{-4}$ and weight-decay of $10^{-3}$ for 300 epochs.

\vspace{5pt}
\noindent \textbf{Experimental Results:} We evaluate performance of the predictor and adversary as we vary the trade-off parameter $\alpha$. We first note that, ideally, the desired predictor accuracy is 100\%, adversary accuracy is 10\% (random chance for 10 classes) and adversary entropy is 2.3 nats (uniform distribution for 10 classes). Figure~\ref{fig:cifar_nondominated} (a)-(b) shows the trade-off achieved between predictor and adversary along with the corresponding normalized hyper-volume (HV). For the predictor and adversary accuracy, the HV corresponds to area above the trade-off curve, while for the predictor accuracy and adversary entropy the HV is the area under the curve. 

We obtain these results by repeating all the experiments five times and retaining the non-dominated solutions i.e., a solution that is no worse than any other solution in both the objectives. From these results, we observe that without privacy considerations, the representation achieves the best target accuracy but also leaks significant information. In contrast adversarial learning of the representation achieves a better trade-off between utility and information leakage. Among ARL approaches, we observe that MaxEnt-ARL is able to obtain a better trade-off compared to ML-ARL. Furthermore, among all possible solutions, MaxEnt-ARL achieves the solution closest to the ideal desired point.

\begin{figure*}[t]
    \centering
    \begin{subfigure}[t]{0.246\textwidth}
        \centering
        \includegraphics[width=\textwidth]{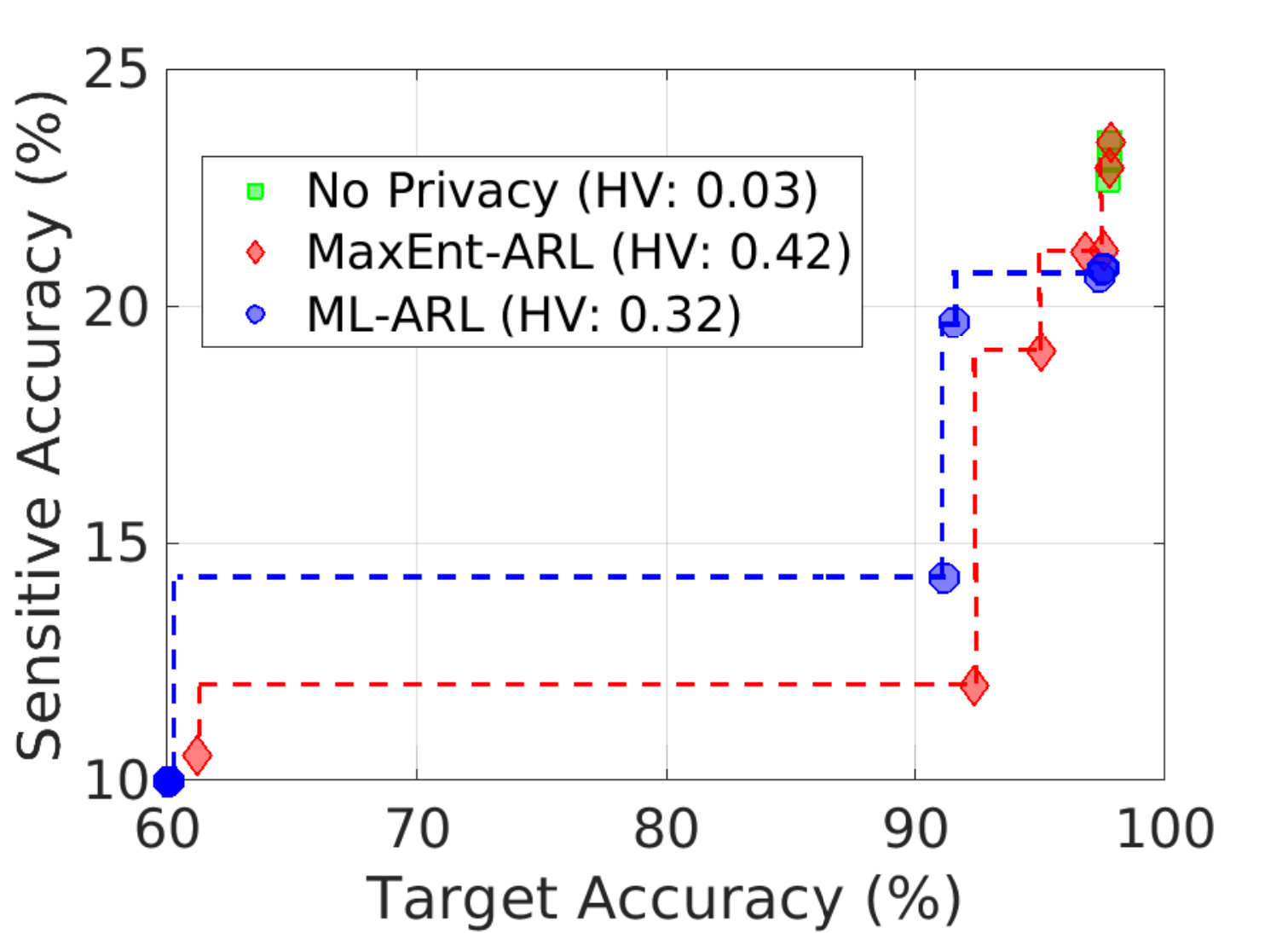}
        \caption{}
    \end{subfigure}
    \begin{subfigure}[t]{0.246\textwidth}
        \centering
        \includegraphics[width=\textwidth]{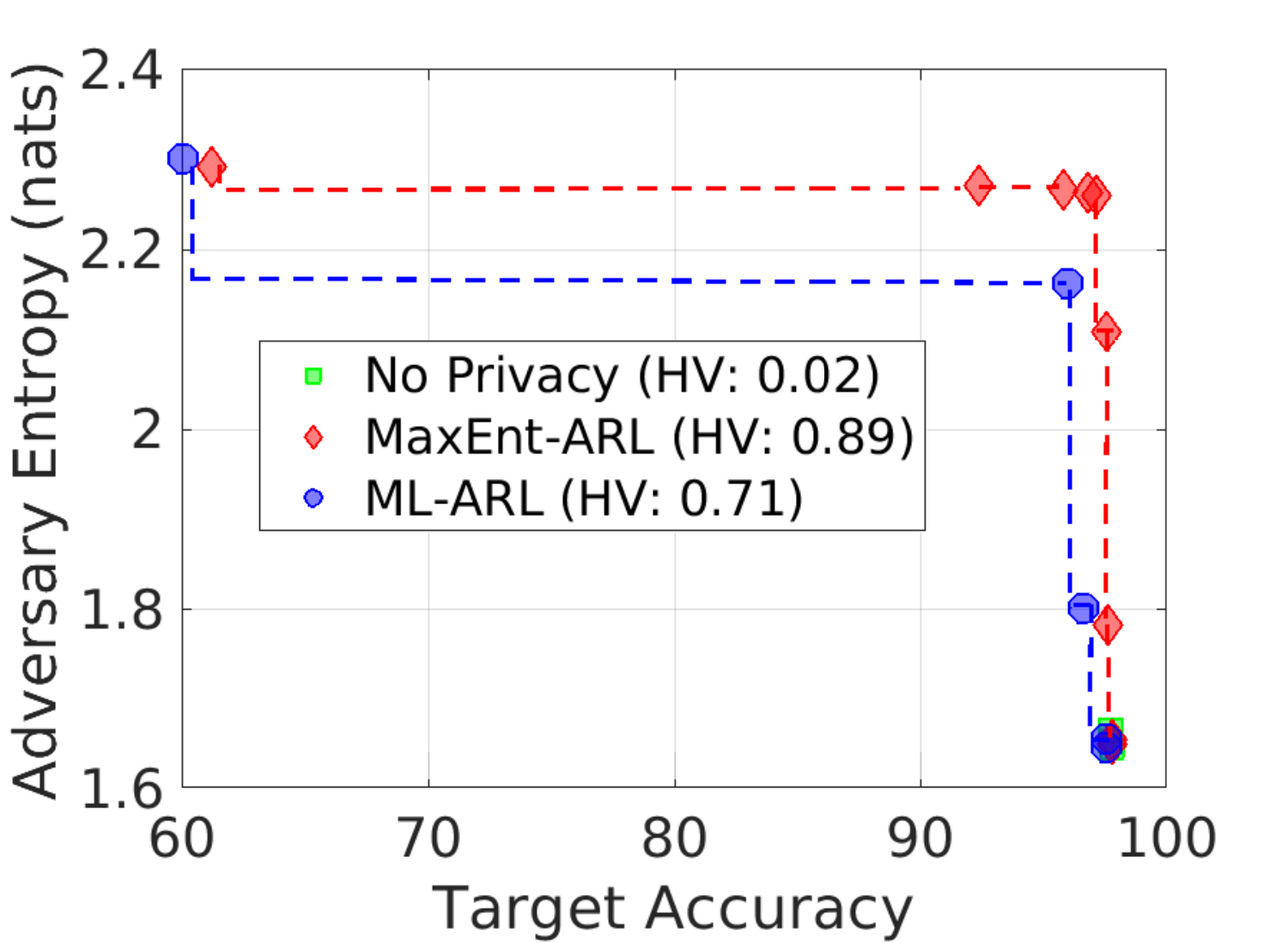}
        \caption{}
    \end{subfigure}
    \begin{subfigure}[t]{0.246\textwidth}
        \centering
        \includegraphics[width=\textwidth]{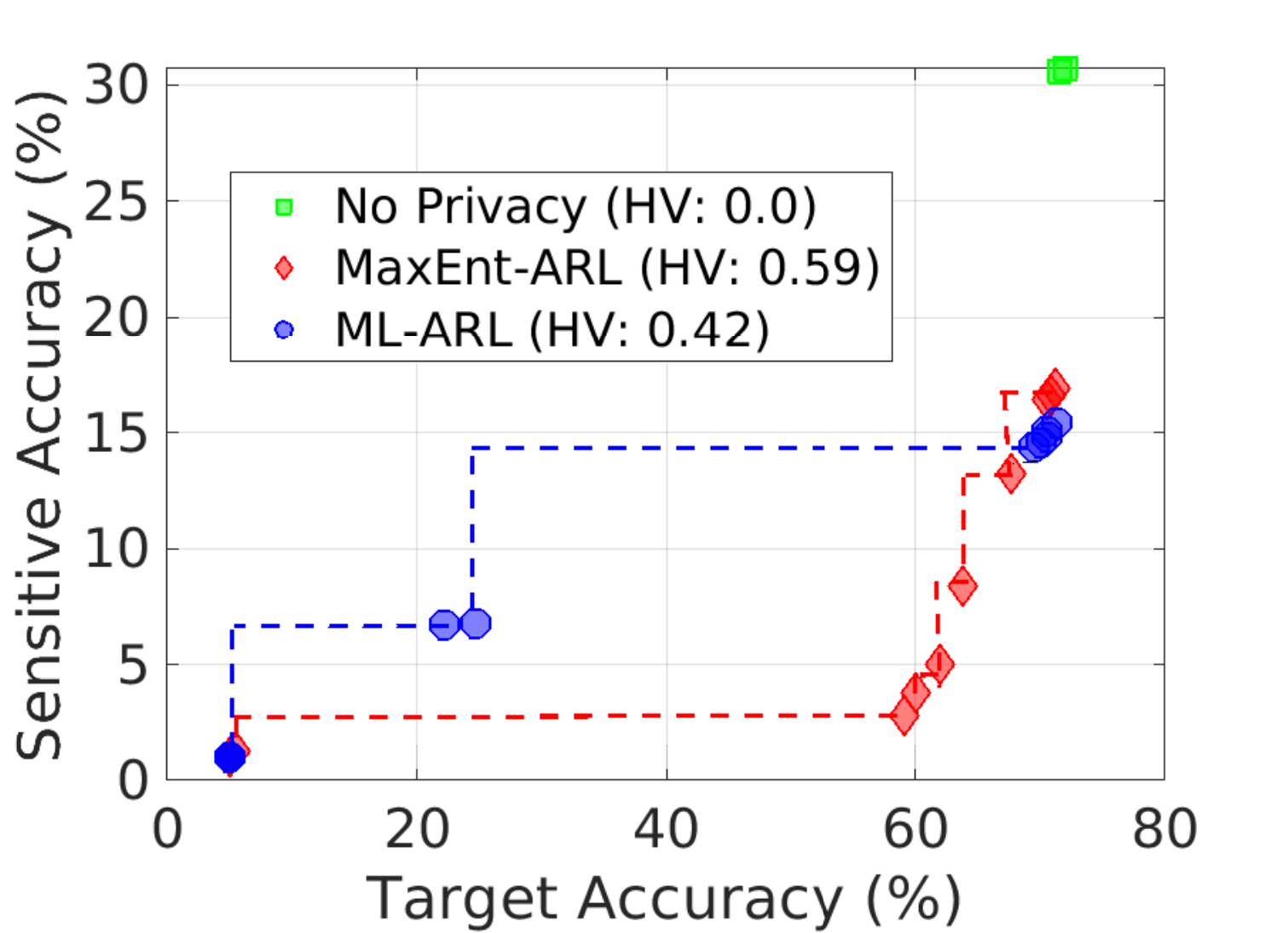}
        \caption{}
    \end{subfigure}
    \begin{subfigure}[t]{0.246\textwidth}
        \centering
        \includegraphics[width=\textwidth]{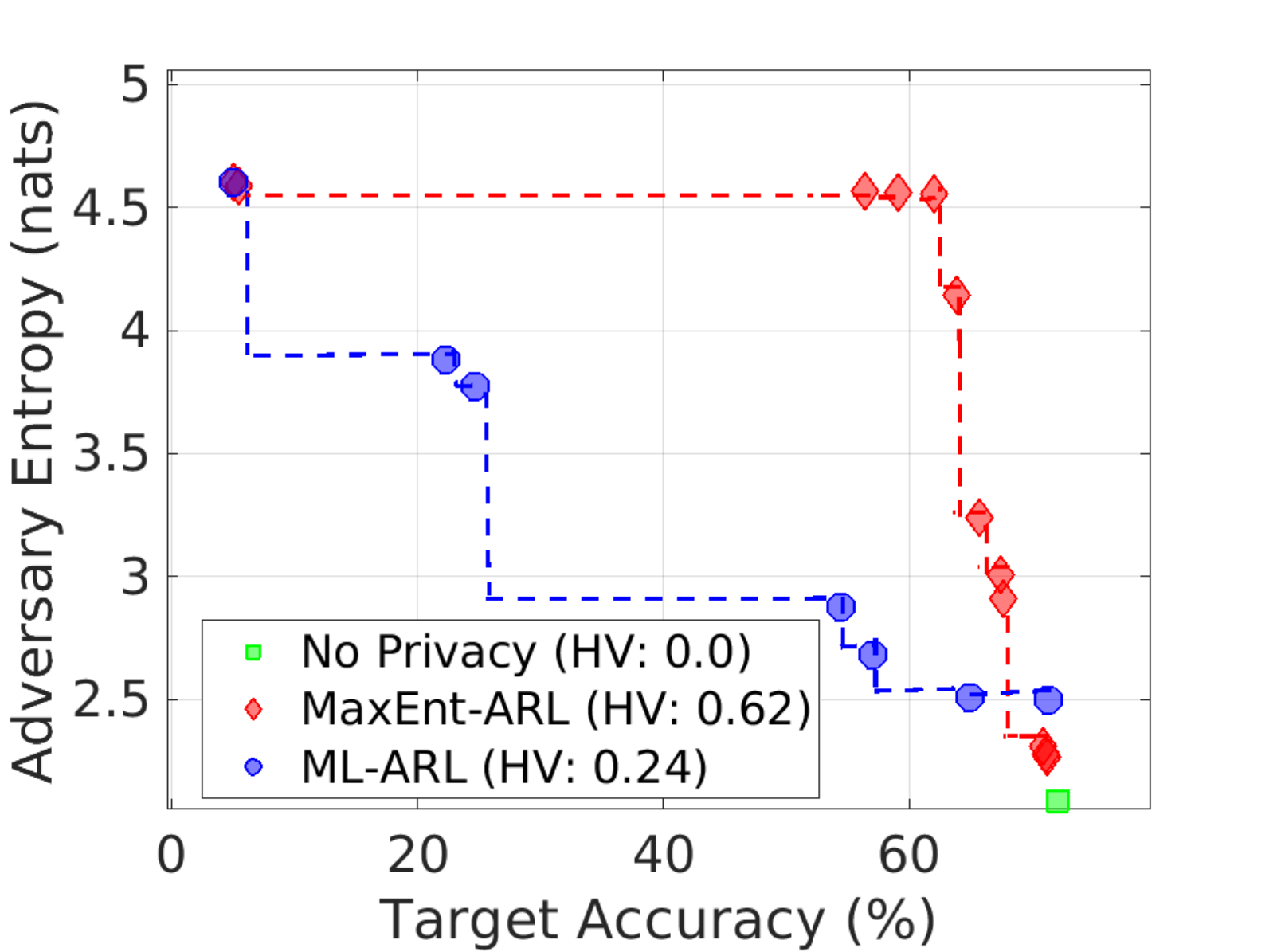}
        \caption{}
    \end{subfigure}
    \caption{\textbf{Adversary Representation Learning on CIFAR Datasets:} Trade-off fronts for two different ARL approaches, ML-ARL and MaxEnt-ARL, in comparison to standard no privacy representation learning. Plots (a)-(b) and (c)-(d) correspond to CIFAR-10 and CIFAR-100 experiments, respectively. In (a) and (c) the ideal desired solution is the bottom right corner, while in (b) and (d) it is the top right corner. HV in the legend corresponds to normalized hyper-volume. Exact numerical values are available in the supplementary material. \label{fig:cifar_nondominated}}
\end{figure*}

\subsection{CIFAR-100}
\begin{table}[t]
\caption{Main classes and Superclasses in CIFAR-100\label{table:cifar100summary}}
\centering
\scalebox{0.6}{
\begin{tabular}{|c|c|}\hline
Superclass & Main Class\\ \hline \hline
aquatic mammals &	beaver, dolphin, otter, seal, whale \\ \hline
fish &	aquarium fish, flatfish, ray, shark, trout \\ \hline
flowers &	orchids, poppies, roses, sunflowers, tulips \\ \hline
food containers &	bottles, bowls, cans, cups, plates\\ \hline
fruit and vegetables &	apples, mushrooms, oranges, pears, sweet peppers\\ \hline
household electrical devices &	clock, computer keyboard, lamp, telephone, television\\ \hline
household furniture &	bed, chair, couch, table, wardrobe\\ \hline
insects &	bee, beetle, butterfly, caterpillar, cockroach\\ \hline
large carnivores &	bear, leopard, lion, tiger, wolf\\ \hline
large man-made outdoor things & 	bridge, castle, house, road, skyscraper\\ \hline
large natural outdoor scenes &	cloud, forest, mountain, plain, sea\\ \hline
large omnivores and herbivores &	camel, cattle, chimpanzee, elephant, kangaroo\\ \hline
medium-sized mammals &	fox, porcupine, possum, raccoon, skunk\\ \hline
non-insect invertebrates &	crab, lobster, snail, spider, worm\\ \hline
people &	baby, boy, girl, man, woman\\ \hline
reptiles &	crocodile, dinosaur, lizard, snake, turtle\\ \hline
small mammals & 	hamster, mouse, rabbit, shrew, squirrel\\ \hline
trees &	maple, oak, palm, pine, willow\\ \hline
vehicles 1 &	bicycle, bus, motorcycle, pickup truck, train\\ \hline
vehicles 2 &	lawn-mower, rocket, streetcar, tank, tractor\\ \hline
\hline
\end{tabular}}
\end{table}

We formulate a new privacy problem on the CIFAR-100 dataset. The dataset consists of 100 classes and are grouped into 20 superclasses (Table \ref{table:cifar100summary}). Each image has a ``fine'' (the class to which it belongs) and a ``coarse" (the superclass to which it belongs) label. We treat the ``coarse" (superclass) and ``fine" (class) labels as the target and sensitive attribute, respectively. So the encoder is tasked to learn features of the super-classes while not revealing the information of the underlying classes. We adopt ResNet-18 as the encoder while the predictor, discriminator and adversary are all 2-layered fully connected networks. The adversarial game is trained for 150 epochs, followed by training the adversary for 100 epochs while the parameters of the encoder are frozen. 

Just as in the case of CIFAR-10, we report the trade-off achieved between predictor and adversary along with the corresponding normalized hyper-volume (HV) in Fig. \ref{fig:cifar_nondominated} (c)-(d). Here we note that, ideally, we desire predictor accuracy of 100\%, adversary accuracy of 1\% (random chance for 100 classes) and adversary entropy of $\ln{100}=4.61$ nats (uniform distribution for 100 classes). We make the following observations from the results. Firstly, the performance of the different approaches suggest that this task is significantly harder than the CIFAR-10 task, with much lower achievable target accuracy and much higher adversary accuracy. Secondly, representation learning without privacy considerations leaks significant amount of information. Thirdly, MaxEnt-ARL is able to significantly outperform ML-ARL on this task, achieving trade-off solutions that are far better, both in terms of adversary accuracy and entropy of adversary.

\section{Conclusion}

This paper introduced a new formulation of \emph{Adversarial Representation Learning} called \emph{Maximum Entropy Adversarial Representation Learning} (MaxEnt-ARL) for mitigating information leakage from learned representations under an adversarial setting. In this model, the encoder is optimized to maximize the entropy of the adversary's distribution of a sensitive attribute as opposed to minimizing the likelihood (ML-ARL) of the true sensitive label. We analyzed the equilibrium and convergence properties of the ML-ARL and MaxEnt-ARL. Numerical experiments on multiple datasets suggests that MaxEnt-ARL is a promising framework for preventing information leakage from image representations, outperforming the baseline minimum likelihood objective.

\section{Appendix}

In this appendix we include proof of Theorem 1 in Section \ref{sec:proof-th1}, Corollary 1.1 in Section \ref{sec:proof-th2} and finally provide the numerical values of the trade-off fronts in the CIFAR-10 and CIFAR-100 experiment in Section \ref{sec:cifar}.

\subsection{Proof of Theorem 1 \label{sec:proof-th1}}
\begin{theorem}\label{th1}
Given a fixed encoder $E$, the optimal discriminator is $q_D(s|E(\bm{x};\bm{\theta}_E);\bm{\theta}_D^{*})=p(s|E(\bm{x};\bm{\theta}_E))$ and the optimal predictor is $q_T(t|E(\bm{x};\bm{\theta}_E);\bm{\theta}_T^{*})=p(t|E(\bm{x};\bm{\theta}_E))$.
\end{theorem}
\begin{proof}
\noindent Let, $\bm{z}$ be the fixed encoder output from input $\bm{x}$ i.e.  $\bm{z}=E(\bm{x};\bm{\theta}_E)$. Let, $p(\bm{x},t,s)$ be the true joint distribution of the variables, i.e. input $\bm{x}$, target label $t$ and sensitive label $s$. The fixed encoder is a deterministic transformation of $\bm{x}$ and generates an implicit distribution $p(\bm{z},t,s)$. 

\vspace{5pt}
\noindent\emph{Discriminator:} The objective of the discriminator is,
\begin{equation}
    \begin{aligned}
    V_1(\bm{\theta}_E,\bm{\theta}_D) &= KL\left(p\left(s|\bm{x}\right)\|q_D\left(s|E(\bm{x};\bm{\theta}_E);\bm{\theta}_D\right)\right)\\
    &=\mathop{\mathbb{E}}_{(\bm{z},t,s)\sim p(\bm{z},t,s)}~ -\log{q_D(s|\bm{z};\bm{\theta}_D)}\\
    &=-\sum_{\bm{x},t,s} p(\bm{x},t,s) \log{q_D(s|\bm{z};\bm{\theta}_D)}\\
    \mbox{s.t.}~~&\sum_{s}q_D(s|\bm{z};\bm{\theta}_D)=1,~~\forall \bm{z}\\
    &~~~~~~~q_D(s|\bm{z};\bm{\theta}_D)\geq 0,~~\forall \bm{z}
    \end{aligned}
\end{equation}
The Lagrangian dual of the problem can be written as
\begin{align*}
    L=V_1(\bm{\theta}_E,\bm{\theta}_D) + &\sum_{\bm{z}} \lambda(\bm{z})\left( \sum_{s}q_D(s|\bm{z};\bm{\theta}_D)-1\right)
\end{align*} 
Now we take partial derivative of $L$ w.r.t. $q_D(s|\bm{z};\bm{\theta}_D^{*})$, the distribution of optimal discriminator. Therefore, the optimal discriminator satisfies,
\begin{equation}
\begin{aligned}
    \frac{\partial L}{\partial q_D(s|\bm{z};\bm{\theta}_D^{*})}&=~0\\
    \Rightarrow -\frac{\sum_{t} p(\bm{z},t,s)} {q_D(s|\bm{z};\bm{\theta}_D^{*})}+\lambda(\bm{z})&=0\\
    \Rightarrow \lambda(\bm{z})q_D(s|\bm{z};\bm{\theta}_D^{*}) &= p(\bm{z},s)
\end{aligned}
\end{equation}
where we used the fact that, $\sum_{t} p(\bm{z},t,s)=p(\bm{z},s)$. Now summing w.r.t. to variable $s$ on the both sides of last line and using the fact that $\sum_{s}q_D(s|\bm{z};\bm{\theta}_D^{*})=1$ we get,
\begin{equation}
    \lambda(\bm{z})=p(\bm{z}) \nonumber
\end{equation}
By substituting $\lambda(\bm{z})$ we obtain the solution for the optimal discriminator,
\begin{equation}
    q_D(s|\bm{z};\bm{\theta}_D^{*})=\frac{p(\bm{z},s)}{p(\bm{z})} = p(s|\bm{z})
\end{equation}
Therefore, \[q_D(s|E(\bm{x};\bm{\theta}_E);\bm{\theta}_D^{*})= p(s|E(\bm{x};\bm{\theta}_E))\]

\vspace{5pt}
\noindent\emph{Target Predictor:} The objective of the predictor is,
\begin{equation}
    \begin{aligned}
    V_2(\bm{\theta}_E,\bm{\theta}_T) &= KL\left(p\left(t|\bm{x}\right)\|q_T\left(t|E(\bm{x};\bm{\theta}_E);\bm{\theta}_T\right)\right)\\
    &=\mathop{\mathbb{E}}_{(\bm{z},t,s)\sim p(\bm{z},t,s)}~ -\log{q_T(t|\bm{z};\bm{\theta}_T)}\\
    &=-\sum_{\bm{x},t,s} p(\bm{x},t,s) \log{q_T(t|\bm{z};\bm{\theta}_T)}\\
    \mbox{s.t.}~~&\sum_{t}q_T(t|\bm{z};\bm{\theta}_T)=1,~~\forall \bm{z}\\
    &~~~~~~~q_T(t|\bm{z};\bm{\theta}_T)\geq 0,~~\forall \bm{z}
    \end{aligned}
\end{equation}
The Lagrangian dual of the problem can be written as
\begin{align*}
    L=V_2(\bm{\theta}_E,\bm{\theta}_T) + &\sum_{\bm{z}} \lambda(\bm{z})\left( \sum_{t}q_T(t|\bm{z};\bm{\theta}_T)-1\right)
\end{align*} 
Now we take partial derivative of $L$ w.r.t. $q_T(t|\bm{z};\bm{\theta}_T^{*})$, the distribution of optimal predictor. The optimal predictor satisfies the equation.
\begin{equation}
\begin{aligned}
    \frac{\partial L}{\partial q_T(t|\bm{z};\bm{\theta}_T^{*})}&=~0\\
    \Rightarrow -\frac{\sum_{s} p(\bm{z},t,s)} {q_T(t|\bm{z};\bm{\theta}_T^{*})}+\lambda(\bm{z})&=0\\
    \Rightarrow \lambda(\bm{z})q_T(t|\bm{z};\bm{\theta}_T^{*}) &= p(\bm{z},t)
\end{aligned}
\end{equation}
where we used the fact that, $\sum_{s} p(\bm{z},t,s)=p(\bm{z},t)$. Now summing w.r.t. to variable $t$ on the both sides of last line and using the fact that $\sum_{t}q_T(t|\bm{z};\bm{\theta}_T^{*})=1$ we get,
\begin{equation}
    \lambda(\bm{z})=p(\bm{z}) \nonumber
\end{equation}
By substituting $\lambda(\bm{z})$ we obtain the solution of the optimal discriminator
\begin{equation}
    q_T(t|\bm{z};\bm{\theta}_T^{*})=\frac{p(\bm{z},t)}{p(\bm{z})} = p(t|\bm{z})
\end{equation}
Therefore, \[q_T(t|E(\bm{x};\bm{\theta}_E);\bm{\theta}_T^{*})= p(t|E(\bm{x};\bm{\theta}_E))\]
\end{proof} 

\subsection{Proof of Corollary 1.1 \label{sec:proof-th2}}
\begin{corollary}\label{th2} 
When $s \perp \!\!\! \perp t$, let the optimum discriminator and predictor for an encoder $E$ be $q_D(s|E(\bm{x};\bm{\theta}_E);\bm{\theta}_D^{*})$ and $q_T(t|E(\bm{x};\bm{\theta}_E);\bm{\theta}_T^{*})$ respectively. The optimal encoder $E(\cdot)$ in the MaxEnt-ARL formulation induces a uniform distribution in the discriminator $q_D(s|E(\bm{x};\bm{\theta}_E^{*});\bm{\theta}_D^{*})$ over the classes of the sensitive attribute.
\end{corollary}

\begin{proof}
\noindent Here we will prove that, when discriminator is fixed, then the encoder learns a representation of data $\bm{x}$ such that  $q_D(s|E(\bm{x};\bm{\theta}_E^{*});\bm{\theta}_D^{*})=1/m$. First we note that although the discriminator is fixed, the discriminator probability $q_D(s|E(\bm{x};\bm{\theta}_E);\bm{\theta}_{D}^{*})$ can change by changing the encoder parameters $\bm{\theta}_E$. Optimization of the encoder in MaxEnt-ARL is formulated as:
\begin{equation}
\label{eq:obj_func_entropy}
\begin{aligned}
&\min~V = \min_{\bm{\theta}_E} \mbox{ } \mathop{\mathbb{E}}_{(\bm{x},t,s)\sim p(\bm{x},t,s)}\left[-\log q_T(t|E(\bm{x};\bm{\theta}_E);\bm{\theta}_T^{*})\right]\\ 
&+\alpha\mathop{\mathbb{E}}_{\bm{x}}\left[\sum_{i=1}^m q_D(s_i|E(\bm{x};\bm{\theta}_E);\bm{\theta}_D^{*})\log q_D(s_i|E(\bm{x};\bm{\theta}_E);\bm{\theta}_D^{*})\right] \\
&+ \log m \\
&\mbox{s.t.}~~\sum^m_{i=1}q_D(s_i|E(\bm{x};\bm{\theta}_E);\bm{\theta}_D^{*})=1\\
&~~~~~~~~~~~~~~q_D(s_i|E(\bm{x};\bm{\theta}_E);\bm{\theta}_D^{*})\geq 0,~~\forall i
\end{aligned}
\end{equation}

The Lagrangian dual of the problem can be written as,
\begin{align*}
    L=V - \lambda \left(\sum^m_{i=1}q_D(s_i|E(\bm{x};\bm{\theta}_E);\bm{\theta}_D^{*})-1\right)
\end{align*} 
Here $\lambda$ is a Lagrangian multiplier and is assumed to be a constant in the absence of any further information. Since $s \perp \!\!\! \perp t$, we have  $q_T(t|E(\bm{x};\bm{\theta}_E);\bm{\theta}_T^{*})$ is independent of $q_D(s|E(\bm{x};\bm{\theta}_E);\bm{\theta}_D^{*})$ given $E(\bm{x};\bm{\theta}_E)$ from Theorem \ref{th1}. Therefore, if we take derivative of $L$ w.r.t. $q_D(s_i|E(\bm{x};\bm{\theta}_E);\bm{\theta}_D^{*})$ and set it to zero we have:
\begin{equation}
    \begin{aligned}
    \frac{\partial L}{\partial q_D(s_i|E(\bm{x};\bm{\theta}_E);\bm{\theta}_D^{*})}&=~0\\
    \Rightarrow 1+\log{(q_D(s_i|E(\bm{x};\bm{\theta}_E);\bm{\theta}_D^{*})}-\lambda&=0\\
    \Rightarrow q_D(s_i|E(\bm{x};\bm{\theta}_E);\bm{\theta}_D^{*})&=\exp{(\lambda-1)}
    \end{aligned}
\end{equation}
Using the first (non-trivial) constraint, we have
\begin{align*}
    \sum^m_{i=1} q_D(s_i|E(\bm{x};\bm{\theta}_E);\bm{\theta}_D^{*})&=1\\
    \sum^m_{i=1} \exp{(\lambda-1)}&=1\\
    \exp{(\lambda-1)}\sum^m_{i=1}&1=1\\
    m(\exp{(\lambda-1)})&=1\\
    \lambda &= \log{\left(1/m\right)}+1\
\end{align*}
Hence, the probability distribution of the discriminator after the encoder's parameters $\bm{\theta}_E$ are optimized is $q_D(s_i|E(\bm{x};\bm{\theta}_E^{*});\bm{\theta}_D^{*})=1/m$. Thus, when the optimum discriminator parameters are fixed, the encoder optimizes the representation such that the discriminator does not leak any information, i.e., it induces a uniform distribution. 
\end{proof}

\subsection{CIFAR Trade-Off\label{sec:cifar}}

We report the numerical values of the target accuracy and adversary accuracy trade-off results on the CIFAR-10 and CIFAR-100 experiments in Table \ref{tab:cifar10_ll} and Table \ref{tab:cifar100_ll}, respectively. Similarly, we report the numerical values of the target accuracy and adversary entropy trade-off results on the CIFAR-10 and CIFAR-100 experiments in Table \ref{tab:cifar10_le} and Table \ref{tab:cifar100_le}, respectively.

\begin{table*}
\begin{subtable}{.5\linewidth}\centering
{\scalebox{0.8}{\begin{tabular}{c|ccc}
\hline
Target Accuracy (\%) & 97.75 & 97.73 & 97.68 \\
\hline
Adversary Accuracy (\%) & 23.44 & 23.09 & 22.68 \\
\hline
\end{tabular}}}
\caption{No Privacy}\label{tab:cifar10_ll_np}
\end{subtable}%
\begin{subtable}{.5\linewidth}\centering
{\scalebox{0.8}{\begin{tabular}{c|cccccc}
\hline
Target Accuracy (\%) & 97.52 & 97.44 & 97.35 & 91.52 & 91.15 & 60.00 \\
\hline
Adversary Accuracy (\%) & 20.83 & 20.77 & 20.64 & 19.68 & 14.27 & 10.00 \\
\hline
\end{tabular}}}
\caption{ML-ARL}\label{tab:cifar10_ll_ml}
\end{subtable}
\begin{subtable}{\linewidth}\centering
{\begin{tabular}{c|ccccccc}
\hline
Target Accuracy (\%) & 97.78 & 97.74 & 97.53 & 96.79 & 95.01 & 92.34 & 61.17 \\
\hline
Adversary Accuracy (\%) & 23.44 & 22.91 & 21.17 & 21.14 & 19.05 & 12.00 & 10.64 \\
\hline
\end{tabular}}
\caption{MaxEnt-ARL}\label{tab:cifar10_ll_maxent}
\end{subtable}
\caption{CIFAR-10: Target Accuracy (\%) vs Adversary Accuracy\label{tab:cifar10_ll}}
\end{table*}

\begin{table*}
\begin{subtable}{.5\linewidth}\centering
{\scalebox{0.8}{\begin{tabular}{c|ccc}
\hline
Target Accuracy (\%) & 97.75 & 97.73 & 97.71 \\
\hline
Adversary Entropy (nats) & 1.65 & 1.65 & 1.67 \\
\hline
\end{tabular}}}
\caption{No Privacy}\label{tab:cifar10_le_np}
\end{subtable}%
\begin{subtable}{.5\linewidth}\centering
{\scalebox{0.8}{\begin{tabular}{c|ccccc}
\hline
Target Accuracy (\%) & 97.52 & 97.50 & 96.58 & 95.97 & 60.00 \\
\hline
Adversary Entropy (nats) & 1.65 & 1.66 & 1.80 & 2.16 & 2.30 \\
\hline
\end{tabular}}}
\caption{ML-ARL}\label{tab:cifar10_le_ml}
\end{subtable}
\begin{subtable}{\linewidth}\centering
{\begin{tabular}{c|ccccccccc}
\hline
Target Accuracy (\%) & 97.78 & 97.74 & 97.58 & 97.53 & 97.14 & 96.79 & 95.76 & 92.34 & 61.17 \\
\hline
Adversary Entropy (nats) & 1.65 & 1.66 & 1.78 & 2.11 & 2.26 & 2.26 & 2.27 & 2.27 & 2.29 \\
\hline
\end{tabular}}
\caption{MaxEnt-ARL}\label{tab:cifar10_le_maxent}
\end{subtable}
\caption{CIFAR-10: Target Accuracy (\%) vs Adversary Entropy\label{tab:cifar10_le}}
\end{table*}

\begin{table*}
\begin{subtable}{.4\linewidth}\centering
{\scalebox{0.8}{\begin{tabular}{c|cc}
\hline
Target Accuracy (\%) & 71.99 & 71.56 \\
\hline
Adversary Accuracy (\%) & 30.69 & 30.59 \\
\hline
\end{tabular}}}
\caption{No Privacy}\label{tab:cifar100_ll_np}
\end{subtable}%
\begin{subtable}{.6\linewidth}\centering
{\scalebox{0.8}{\begin{tabular}{c|cccccccc}
\hline
Target Accuracy (\%) & 71.32 & 70.52 & 70.43 & 69.98 & 69.42 & 24.66 & 22.22 & 5.00 \\
\hline
Adversary Accuracy (\%) & 15.43 & 15.09 & 14.84 & 14.60 & 14.41 & 6.81 & 6.72 & 1.00 \\
\hline
\end{tabular}}}
\caption{ML-ARL}\label{tab:cifar100_ll_ml}
\end{subtable}
\begin{subtable}{\linewidth}\centering
{\begin{tabular}{c|cccccccccc}
\hline
Target Accuracy (\%) & 71.17 & 70.80 & 70.50 & 67.63 & 63.81 & 61.98 & 60.03 & 59.11 & 5.37 & 5.00 \\
\hline
Adversary Accuracy (\%) & 16.88 & 16.60 & 16.43 & 13.23 & 8.38 & 5.02 & 3.80 & 2.81 & 1.23 & 1.00 \\
\hline
\end{tabular}}
\caption{MaxEnt-ARL}\label{tab:cifar100_ll_maxent}
\end{subtable}
\caption{CIFAR-100: Target Accuracy (\%) vs Adversary Accuracy\label{tab:cifar100_ll}}
\end{table*}

\begin{table*}
\begin{subtable}{.3\linewidth}\centering
{\scalebox{0.8}{\begin{tabular}{c|c}
\hline
Target Accuracy (\%) & 71.99 \\
\hline
Adversary Entropy (nats) & 2.09 \\
\hline
\end{tabular}}}
\caption{No Privacy}\label{tab:cifar100_le_np}
\end{subtable}%
\begin{subtable}{.7\linewidth}\centering
{\scalebox{0.8}{\begin{tabular}{c|ccccccc}
\hline
Target Accuracy (\%) & 71.32 & 64.90 & 56.99 & 54.46 & 24.66 & 22.22 & 5.00 \\
\hline
Adversary Entropy (nats) & 2.50 & 2.51 & 2.68 & 2.88 & 3.77 & 3.88 & 4.60 \\
\hline
\end{tabular}}}
\caption{ML-ARL}\label{tab:cifar100_le_ml}
\end{subtable}
\begin{subtable}{\linewidth}\centering
{\begin{tabular}{c|cccccccccccc}
\hline
Target Accuracy (\%) & 71.17 & 71.05 & 70.80 & 67.63 & 67.38 & 65.71 & 63.81 & 61.98 & 59.11 & 56.32 & 5.37 & 5.00 \\
\hline
Adversary Entropy (nats) & 2.27 & 2.28 & 2.31 & 2.91 & 3.01 & 3.24 & 4.14 & 4.56 & 4.57 & 4.57 & 4.59 & 4.60 \\
\hline
\end{tabular}}
\caption{MaxEnt-ARL}\label{tab:cifar100_le_maxent}
\end{subtable}
\caption{CIFAR-100: Target Accuracy vs Adversary Entropy\label{tab:cifar100_le}}
\end{table*}

{\small
\bibliographystyle{ieee}
\bibliography{egbib}

\begin{thebibliography}{10}\itemsep=-1pt

\bibitem{beutel2017data}
A.~Beutel, J.~Chen, Z.~Zhao, and E.~H. Chi.
\newblock Data decisions and theoretical implications when adversarially
  learning fair representations.
\newblock {\em arXiv preprint arXiv:1707.00075}, 2017.

\bibitem{Dua:2019}
D.~Dua and C.~Graff.
\newblock {UCI} machine learning repository, 2017.

\bibitem{edwards2015censoring}
H.~Edwards and A.~J. Storkey.
\newblock Censoring representations with an adversary.
\newblock In {\em International Conference on Learning Representations (ICLR)},
  2016.

\bibitem{ganin2015unsupervised}
Y.~Ganin and V.~Lempitsky.
\newblock Unsupervised domain adaptation by backpropagation.
\newblock In {\em International Conference on Machine Learning (ICML)}, 2015.

\bibitem{ganin2016domain}
Y.~Ganin, E.~Ustinova, H.~Ajakan, P.~Germain, H.~Larochelle, F.~Laviolette,
  M.~Marchand, and V.~Lempitsky.
\newblock Domain-adversarial training of neural networks.
\newblock {\em The Journal of Machine Learning Research}, 17(1):2096--2030,
  2016.

\bibitem{georghiades2001few}
A.~S. Georghiades, P.~N. Belhumeur, and D.~J. Kriegman.
\newblock From few to many: Illumination cone models for face recognition under
  variable lighting and pose.
\newblock {\em IEEE Transactions on Pattern Analysis \& Machine Intelligence},
  (6):643--660, 2001.

\bibitem{goodfellow2014generative}
I.~Goodfellow, J.~Pouget-Abadie, M.~Mirza, B.~Xu, D.~Warde-Farley, S.~Ozair,
  A.~Courville, and Y.~Bengio.
\newblock Generative adversarial nets.
\newblock In {\em Advances in Neural Information Processing Systems (NeurIPS)},
  pages 2672--2680, 2014.

\bibitem{he2016identity}
K.~He, X.~Zhang, S.~Ren, and J.~Sun.
\newblock Identity mappings in deep residual networks.
\newblock In {\em European Conference on Computer Vision (ECCV)}, pages
  630--645. Springer, 2016.

\bibitem{khalil1996nonlinear}
H.~K. Khalil.
\newblock Nonlinear systems.
\newblock {\em Printice-Hall Inc}, 1996.

\bibitem{kingma2014adam}
D.~P. Kingma and J.~Ba.
\newblock Adam: {A} method for stochastic optimization.
\newblock {\em arXiv preprint arXiv:1412.6980}, 2014.

\bibitem{kingma2013auto}
D.~P. Kingma and M.~Welling.
\newblock Auto-encoding variational bayes.
\newblock {\em arXiv preprint arXiv:1312.6114}, 2013.

\bibitem{krizhevsky2009learning}
A.~Krizhevsky and G.~Hinton.
\newblock Learning multiple layers of features from tiny images.
\newblock Technical report, Citeseer, 2009.

\bibitem{li2014learning}
Y.~Li, K.~Swersky, and R.~Zemel.
\newblock Learning unbiased features.
\newblock {\em arXiv preprint arXiv:1412.5244}, 2014.

\bibitem{louizos2015variational}
C.~Louizos, K.~Swersky, Y.~Li, M.~Welling, and R.~Zemel.
\newblock The variational fair autoencoder.
\newblock In {\em International Conference on Learning Representations (ICLR)},
  2016.

\bibitem{mescheder2017numerics}
L.~Mescheder, S.~Nowozin, and A.~Geiger.
\newblock The numerics of gans.
\newblock In {\em Advances in Neural Information Processing Systems (NeurIPS)},
  2017.

\bibitem{nagarajan2017gradient}
V.~Nagarajan and J.~Z. Kolter.
\newblock Gradient descent gan optimization is locally stable.
\newblock In {\em Advances in Neural Information Processing Systems (NeurIPS)},
  pages 5585--5595, 2017.

\bibitem{tzeng2017adversarial}
E.~Tzeng, J.~Hoffman, K.~Saenko, and T.~Darrell.
\newblock Adversarial discriminative domain adaptation.
\newblock In {\em IEEE Conference on Computer Vision and Pattern Recognition
  (CVPR)}, 2017.

\bibitem{xie2017controllable}
Q.~Xie, Z.~Dai, Y.~Du, E.~Hovy, and G.~Neubig.
\newblock Controllable invariance through adversarial feature learning.
\newblock In {\em Advances in Neural Information Processing Systems (NeurIPS)},
  2017.

\bibitem{zemel2013learning}
R.~Zemel, Y.~Wu, K.~Swersky, T.~Pitassi, and C.~Dwork.
\newblock Learning fair representations.
\newblock In {\em International Conference on Machine Learning (ICML)}, 2013.

\bibitem{zhang2018mitigating}
B.~H. Zhang, B.~Lemoine, and M.~Mitchell.
\newblock Mitigating unwanted biases with adversarial learning.
\newblock In {\em AAAI/ACM Conference on AI, Ethics, and Society}, 2018.

\end{thebibliography}
}
\end{document}